\documentclass{article}

\usepackage{microtype}
\usepackage{graphicx}
\usepackage{subcaption}
\usepackage{booktabs} %

\usepackage[dvipsnames]{xcolor}

\usepackage{enumitem,wrapfig,adjustbox,multirow}
\usepackage{url}
\usepackage{hyperref}

\usepackage[accepted]{icml2023}   %

\usepackage{amsmath}
\usepackage{amssymb}
\usepackage{mathtools}
\usepackage{amsthm}
\usepackage{thmtools,thm-restate}

\usepackage[capitalize,noabbrev]{cleveref}

\theoremstyle{plain}

\theoremstyle{definition}

\theoremstyle{remark}

\usepackage[textsize=tiny]{todonotes}

\icmltitlerunning{Internet Explorer: Targeted Representation Learning on the Open Web}

\newcommand{\eg}{{\it e.g.}}
\newcommand{\ie}{{\it i.e.}}

\newcommand{\wrt}{{\it w.r.t. }}

\newcommand{\EA}{\end{array}}
\newcommand{\BA}{\begin{array}}

\newcommand{\green}[1]{\textcolor{ForestGreen}{#1}}
\newcommand{\red}[1]{\textcolor{red}{#1}}

\newif\ifchanges
\changestrue

\newcount\ppnum
\newcommand\ppnumber[1]{%
        \ppnum=#1\relax
        \ifnum\ppnum<0
                $-$%
                \ppnum=-\ppnum
        \fi
        \let\pptemp\empty
        \loop\ifnum\ppnum>999
                \count255=\ppnum
                \divide\ppnum by1000
                \count255=\numexpr \count255 - 1000*\ppnum \relax
                \edef\pptemp{,\ifnum\count255<100 0\ifnum\count255<10 0\fi\fi
                             \the\count255 \pptemp}%
        \repeat
        \the\ppnum
        \pptemp
}

\begin{document}

\twocolumn[
\icmltitle{Internet Explorer: Targeted Representation Learning on the Open Web}

\icmlsetsymbol{equal}{*}

\begin{icmlauthorlist}
\icmlauthor{Alexander C. Li}{equal,cmu}
\icmlauthor{Ellis Brown}{equal,cmu}
\icmlauthor{Alexei A. Efros}{ucb}
\icmlauthor{Deepak Pathak}{cmu}
\end{icmlauthorlist}

\icmlaffiliation{cmu}{Carnegie Mellon University}
\icmlaffiliation{ucb}{University of California, Berkeley}

\icmlcorrespondingauthor{Alexander Li}{alexanderli@cmu.edu}
\icmlcorrespondingauthor{Ellis Brown}{ellisbrown@cmu.edu}

\icmlkeywords{Machine Learning, ICML, neural networks, self-supervised learning, computer vision, Internet, exploration}

\vskip 0.3in
]

\printAffiliationsAndNotice{\icmlEqualContribution} %

\begin{abstract}
    Modern vision models typically rely on fine-tuning general-purpose models pre-trained on large, static datasets. These general-purpose models only capture the knowledge within their pre-training datasets, which are tiny, out-of-date snapshots of the Internet---where billions of images are uploaded each day. 
    We suggest an alternate approach: rather than hoping our static datasets transfer to our desired tasks after large-scale pre-training, we propose dynamically utilizing the Internet to quickly train a small-scale model that does extremely well on the task at hand.
    Our approach, called Internet Explorer, explores the web in a self-supervised manner to progressively find relevant examples that improve performance on a desired target dataset. It cycles between searching for images on the Internet with text queries, self-supervised training on downloaded images, determining which images were useful, and prioritizing what to search for next. We evaluate Internet Explorer across several datasets and show that it outperforms or matches CLIP oracle performance by using just a single GPU desktop to actively query the Internet for 30--40 hours. Results, visualizations, videos, and code on our
    website:
    \href{https://internet-explorer-ssl.github.io/}{\url{internet-explorer-ssl.github.io/}}

\end{abstract}

\section{Introduction}
Suppose you have a small dataset and need to train a model for some task, say classification.
A pipeline that has become standard today is to download the latest pre-trained deep network and fine-tune it on your own small dataset. This pre-trained model used to be ImageNet-based~\cite{deng2009imagenet,he2016deep} and now would probably be CLIP~\cite{radford2021learning}. The implicit goal set by the community for such pre-trained models is that they should transfer well to any kind of downstream task not known in advance. This has led to a race to build ultra-large-scale models in terms of computation, model size, and dataset size. But is this goal of building an ``omniscient'' pre-trained model that can work on any future downstream task even feasible? Perhaps not, as our world is continually changing.
Although the size of the pretraining datasets has grown from 1.2M~~\cite{deng2009imagenet} to 5B~\cite{schuhmann2022laion} images, what has not changed at all is their nature: these datasets are curated and, more importantly, \textit{\textbf{static}}. For instance, the portion of ImageNet curated before 2007 has no idea what an iPhone is.
Furthermore, although a few hundred million images represent a staggering quantity of visual data, they are minuscule compared to the entire Internet, where billions of new photos are uploaded every day.
Thus, current static datasets, however big they become,
fail to capture the richness and dynamic nature of the data available on the Internet.
Moreover, as our static datasets grow, they require increasingly inaccessible amounts of compute.

\begin{figure}[t]
\centering
\includegraphics[width=\linewidth]{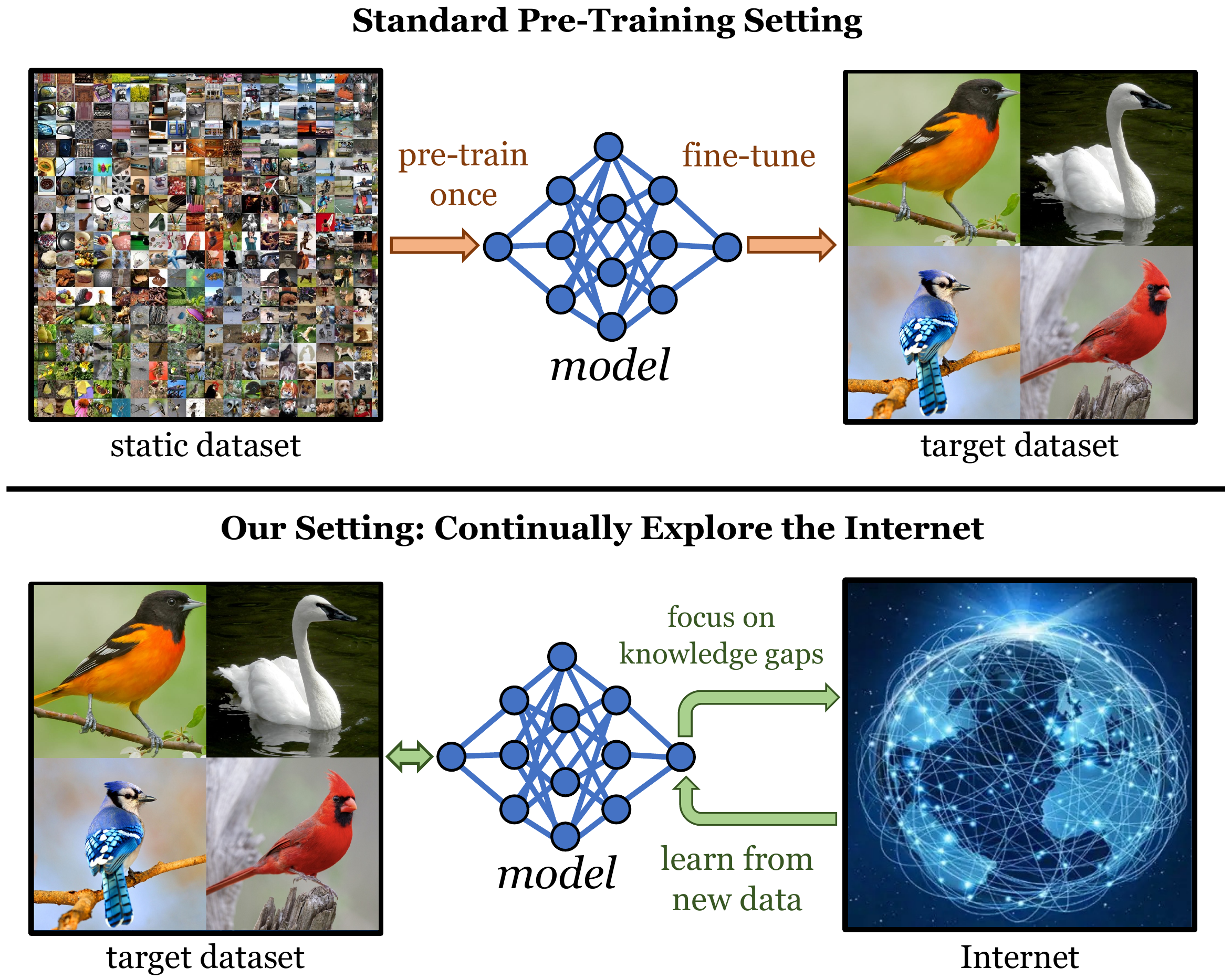}
\vspace{-0.25in}
\caption{Given unlabeled data for a target task, our approach, Internet Explorer, searches the Internet to progressively find more and more relevant training data via self-supervised exploration.}
\label{fig:teaser}
\vspace{-0.15in}
\end{figure}

In this paper, we 
rethink the idea of \textit{\textbf{generic}} large-scale pretraining and propose an alternate paradigm: train a small-scale but up-to-date model geared towards the \textit{\textbf{specific}} downstream task of interest. To do so, we look beyond static datasets and \textit{treat the Internet itself as a dynamic, open-ended dataset}.  
Unlike conventional datasets, which are expensive to expand and grow stale with time, 
the Internet is dynamic, rich, grows automatically, and is always up to date.
Its continuously evolving nature also means we cannot hope to ever download it or train a model, whether large or small, on all of it.

We propose that the Internet can be treated as a special kind of dataset---one that exists out there, ready to be queried as needed to quickly train a customized model for a desired task.
We draw an analogy to reinforcement learning, where even though the task is known, finding a policy that can generate the desired behavior is non-trivial due to the high complexity of the state space. Hence, most approaches rely on some form of exploration to figure out what actions the agent should take so that it quickly finds high-reward states. Inspired by this analogy, we formulate a disembodied, online agent we call {\em Internet Explorer}, that actively queries standard search engines to find relevant visual data that improve feature quality on a target dataset (see \cref{fig:teaser}). The agent's actions are text queries made to search engines, and the observations are the data obtained from the search.

The queries made by Internet Explorer improve over time. It cycles between searching for images on the Internet with text queries, self-supervised training on downloaded images, determining which images are relevant to the target dataset, and prioritizing what to search for next (see \cref{fig:method}). We also bootstrap Internet Explorer using existing pre-trained models such as MoCo-v3~\cite{he2020momentum} and obtain a significant boost on the target datasets.

Our setting is different from active learning~\cite{settles2009active}, where the goal is to selectively obtain labels for data points from a fixed dataset. In contrast, Internet Explorer continually expands the size of its dataset and requires no labels for training, even from the target dataset. 
Some prior works have also discussed ways to leverage the Internet as an additional source of data. NELL~\cite{carlson2010toward} proposed a way to continually scrape web pages to learn new concepts and relationships, which are periodically curated by a human in the loop. NEIL~\cite{chen2013neil} builds on NELL's dictionary to search visual data and develop visual relationships. Both are semi-supervised methods to gather general ``common-sense'' knowledge from the Internet. In contrast, we perform an actively improving \textit{directed} search to perform well on target data, in a fully self-supervised manner. Recent work~\cite{jiang2021improving} follows a similar setting but searches a static dataset and not the Internet.

We evaluate Internet Explorer across 7 datasets, including 4 fine-grained datasets, PASCAL VOC, ImageNet-100, and FMoW-WILDS.
We search for relevant images using Google; however, the method is compatible with any text-based search engine or even a static dataset (see \cref{subsec:search_engine_main}).
We compare against several strong baselines, including CLIP, on downstream tasks. Note that CLIP acts as an oracle for our approach because it has likely already seen all or more queries that Internet Explorer makes.
In most scenarios, Internet Explorer either outperforms or matches the CLIP oracle using only a single 3090 GPU desktop machine that runs for 30--40 hours, makes over 10K progressively improving queries, and downloads over 1M relevant Internet images for each target dataset.

\begin{figure}[t]
    \centering
    \includegraphics[width=\linewidth]{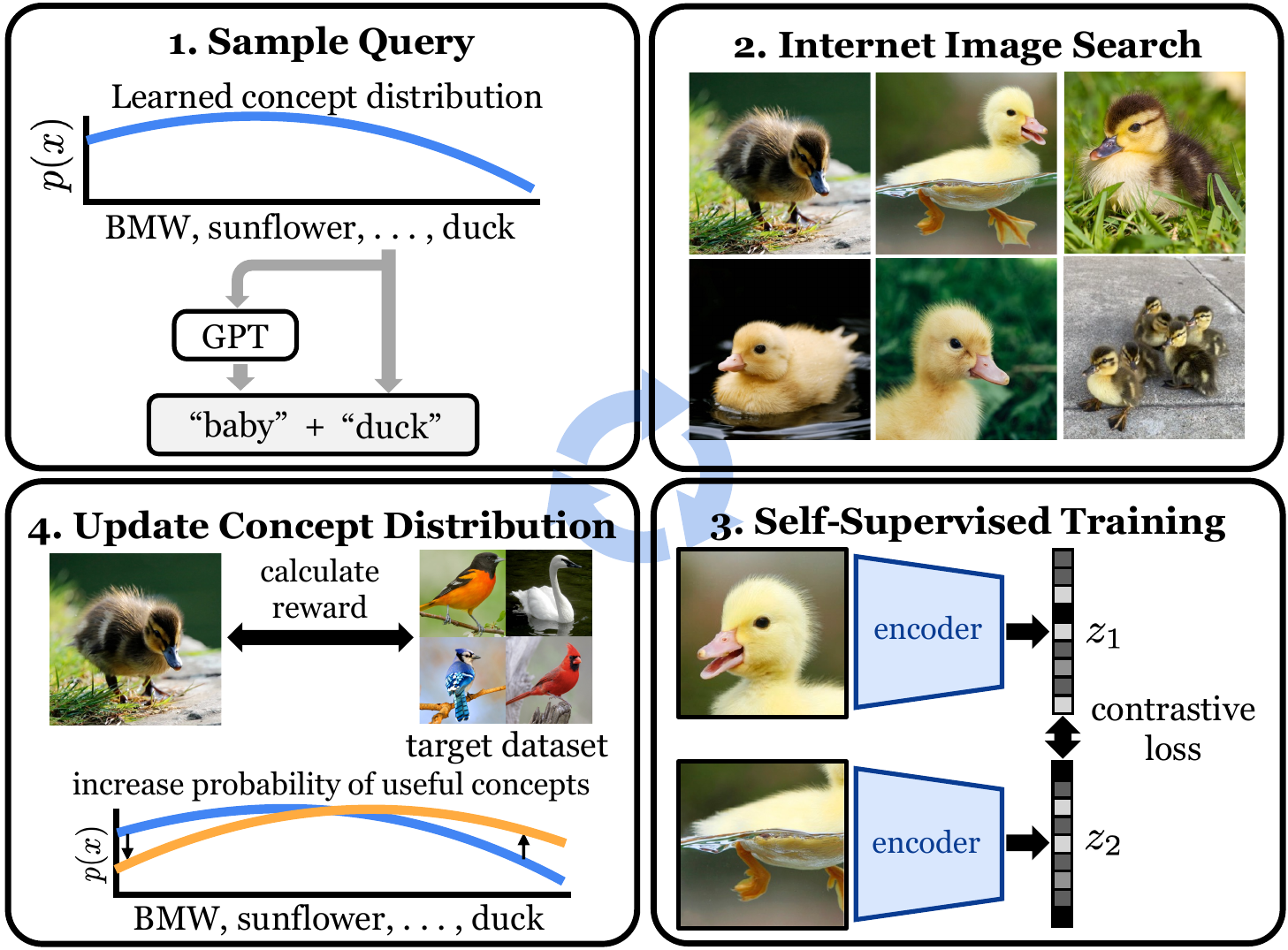}
    \vspace{-2em}
    \caption{\textbf{Overview of Internet Explorer.} Our goal is to efficiently search the Internet for images that improve our performance on a target dataset.
     In each iteration, we first generate text queries by combining a concept sampled from a learned distribution with a GPT-generated descriptor (\S\ref{subsec:text_query_generation}, \S\ref{subsec:tiering}). Next, we query search engines with the resulting phrase and download the top 100 image results (\S\ref{subsec:text_to_image_search}, \ref{subsec:search_engine_main}). We add these images to the set of previously downloaded images and perform self-supervised training on the combined dataset (\S\ref{subsec:ssl}). Finally, we evaluate the relevance of the new images and update our concept distribution to increase the likelihood of similar queries if their images were similar to the target dataset (\S\ref{subsec:image_rel_reward}, \S\ref{subsec:unseen_reward}).
    }
    \label{fig:method}
    \vspace{-1em}
\end{figure}

\section{Internet Explorer: An Online Agent}

We focus on the problem of efficiently improving representations for some target dataset by acquiring Internet data.
We make as few assumptions as possible and assume that we have only unlabeled training data from the target dataset. 
Successful representation learning in this setting would lead to better performance on the target dataset distribution for standard tasks like classification and detection, and potentially others where the labels are not semantic (e.g., depth prediction or robotics).
An overview of the Internet Explorer method is depicted in \cref{fig:method} and described in \cref{alg:internet_explorer}.

\subsection{Text-to-image Search}
\label{subsec:text_to_image_search}
We discover and download images from the full breadth of the Internet by querying text-to-image search engines, which return images based on their captions and surrounding text. Text-to-image search is fast, finds diverse images from across the Internet, and enables searches for vastly different queries simultaneously. Note that text-to-image search is noisy and makes use of weak supervision (the image-text pairing on webpages). Thus, we only perform self-supervised training on the downloaded images. We use a public codebase to query Google Images, which can download the top 100 images for each query~\cite{hardikvasa, Joeclinton1}. We also try other search engines in Section~\ref{subsec:search_engine_main}.

\subsection{Text Query Generation}
\label{subsec:text_query_generation}
As text queries are our only input interface with the Internet, it is crucial that we can generate diverse queries that correspond to a variety of visual categories. Specificity is also important. Once a useful visual category is identified, generating fine-grained variants of the query is necessary to obtain data for all visual variations in the category.
We construct queries by combining two components: 
\begin{enumerate}[noitemsep,topsep=0pt]
    \item \textit{Concepts} specify semantic categories such as people, places, or objects. %
    \item \textit{Descriptors} are modifiers that generate variations in appearance. %
\end{enumerate}

We draw our concepts from the WordNet hierarchy \cite{miller1995wordnet}, which consists of $146{,}347$ noun lemmas. Not all of these lemmas are visual, but the vocabulary still covers an incredible range of topics (see examples in \cref{sec:wordnet_lemmas}).
To generate a text query, we first sample a concept from a learned distribution over our vocabulary. This discrete distribution is defined by our estimates of how relevant each concept in the vocabulary is at the current time (see Section \ref{subsec:image_rel_reward} for details on estimating rewards and Section \ref{subsec:tiering} for the distribution).
Given a sampled concept, we can generate a descriptor by prompting a GPT-J language model \cite{gpt-j} with examples of descriptor-concept pairs (details in \cref{sec:gptj-descriptors}).
Finally, as shown in Step 1 of Figure \ref{fig:method}, we concatenate the concept and descriptor. If our concept is ``duck'' and the GPT-generated descriptor is ``baby,'' our search engine query is ``baby duck.''

\subsection{Self-supervised Training}
\label{subsec:ssl}

We use self-supervised learning (SSL) to learn useful representations from the unlabeled images that we download from the Internet. 
Internet Explorer is compatible with any SSL algorithm that uses images or image-text pairs, including contrastive~\cite{he2020momentum,chen2020simple}, non-contrastive~\cite{grill2020bootstrap,zbontar2021barlow,bardes2021vicreg,caron2021emerging}, masking-based~\cite{bao2021beit,he2022masked}, or multimodal~\cite{radford2021learning} approaches. 
For speed and stability reasons, we use the MoCo-v3 algorithm~\cite{chen2021empirical}, which trains encoders $f_q$ and $f_k$ on augmentations $(x_1, x_2)$ of the same image to output vectors $q = f_q(x_1)$ and $k = f_k(x_2)$. $f_q$ is trained to minimize the InfoNCE loss \cite{oord2018representation}:
\begin{align}
    \mathcal L_q = -\log \frac{\exp(q \cdot k^+ / \tau)}{\exp (q \cdot k^+ / \tau) + \sum_{k^-} \exp (q \cdot k^- / \tau) }
\label{eq:moco_loss}
\end{align}
$k^+$ corresponds to $f_k$'s output on the other augmentation of the image used to compute $q$, and the set of negative examples $\{k^-\}$ corresponds to $f_k$'s output on other images in the batch. The temperature $\tau$ is set to $1$ by default. $f_k$ consists of a base encoder, a projection MLP, and a prediction head, whereas $f_q$ is the exponential moving average of the base encoder and projection MLP from $f_k$. By training $q$ and $k^+$ to be similar across image augmentations, MoCo-v3 encourages the network to learn high-level semantic features.

Before turning to the Internet, we initialize a ResNet-50 model \cite{he2016deep} using a MoCo-v3 checkpoint trained offline for 100 epochs on ImageNet and then fine-tuned on the target dataset. Without using labels, we select the best starting checkpoint by early stopping on the SSL loss, which highly correlates with target accuracy~\cite{li2022understanding}.
In each iteration of our method, we use MoCo-v3 to fine-tune our encoder on a mixture of newly downloaded, previously downloaded, and target dataset images.

\subsection{Image Relevance Reward}
\label{subsec:image_rel_reward}
We want to rank newly downloaded images by how much they improve our features for the target dataset. This allows us to (a) prioritize taking gradient steps on useful images, and (b) understand what to search for in subsequent iterations. Unfortunately, it is challenging to directly measure the effect of an individual training example on performance. Numerous techniques have been proposed \cite{koh2017understanding,feldman2020neural,paul2021deep,ilyas2022datamodels}, but they all require extensive and repeated training on new images to estimate their impact. 

Instead of trying to precisely measure what is learned from each image, we use its similarity to the target dataset as a proxy for being relevant to training.
We rank the downloaded images by their similarity in representation space to the target dataset images; those most similar to the target dataset induce larger contrastive loss since each $\exp(q \cdot k^-)$ term in the denominator of Eq.~\ref{eq:moco_loss} is larger when the negative examples $\{k^-\}$ are closer to $q$.
These ``hard negatives'' \cite{robinson2020contrastive,schroff2015facenet,oh2016deep,harwood2017smart,wu2017sampling,ge2018deep} yield larger and more informative gradients and should result in the biggest improvement in representation quality.
Thus, overloading notation for $k$, we compute the reward for a particular image as its representation's average cosine similarity to its $k$ closest neighbors in the target dataset. Given an image encoder $f_k: \mathbb{R}^{H\times W\times 3} \rightarrow \mathbb{R}^d$, an unlabeled target dataset $\mathcal D = \{ x_i\}_{i=1}^N$, and a new image $y$ to evaluate, the reward is calculated:
\begin{align}
    r(f_k, \mathcal D, y) = \max_{\substack{I \subset \{1, ..., N\}; \\ |I| = k}}\frac{1}{k} \sum_{i \in I} S_{\cos}(f_k(x_i), f_k(y))
\end{align}
where $S_{\cos}$ is the cosine similarity. A previous metric for identifying relevant data~\cite{jiang2021improving} used $k=1$ nearest neighbors, 
but we found that this was too noisy and allowed high rewards for outlier target images to distract our search.
We instead use $k=15$ to improve the accuracy of our relevance estimation.
In \cref{subsec:reward_analysis}, we compare our reward to alternatives and explore their failure modes. This reward is used for two purposes: determining which of the downloaded images to train on and, subsequently, which concepts would be useful to search for next.

\vspace{-0.06in}
\paragraph{Which images to train on.}
Many newly downloaded images are not worth training on, since they come from unrelated queries or are noisy results from the search engine.
Thus, at the end of each iteration, we rank the newly downloaded images by their reward and save the top $50\%$ to a replay buffer that we maintain across iterations. In subsequent iterations, we continue training on this filtered data.

\vspace{-0.06in}
\paragraph{Determining which concepts are useful.}
When we search for a concept and get back $Q$ image results $\{I_i\}_{i=1}^Q$, we take the average of the top 10 image-level rewards $r_i = r(f_k, \mathcal D, I_i)$ and use that as a \textit{concept-level score}. This gives us an accurate measure of the relevance of a particular query and reduces the impact of noisy search results.

\begin{algorithm}[t]
   \caption{$\texttt{Internet Explorer}$}
   \label{alg:internet_explorer}
\begin{algorithmic}[1]
    \STATE {\bfseries Input:} target dataset $\mathcal D$, SSL algorithm $\mathbb{A}$, search engine $\texttt{SE}$, encoder $f: \mathbb{R}^{H \times W \times 3} \rightarrow \mathbb{R}^d$, image reward function $r$, vocabulary $\mathcal V = \{c_i\}_{i=1}^C$, $\#$ concepts/itr $M$, $\#$ query results/search $Q$, 
    GPT-based concept $\rightarrow$ descriptor function $\texttt{GPTDesc}$, 
    concept distribution function $\texttt{CalcProb}$
    \STATE Initialize replay buffer $\mathcal{B} \leftarrow \emptyset$
    \STATE Initialize concept distribution $p = \text{Uniform}\{1, C\}$
    \FOR{iteration $=1, 2, \dots$}
        \FOR{$i = 1, \dots, M$}
            \STATE Sample concept $c_i \sim p(\mathcal{V})$ \hfill (\S\ref{subsec:text_query_generation})
            \STATE Sample descriptor $d_i \gets \texttt{GPTDesc}(c_i)$ \hfill (\S\ref{sec:gptj-descriptors})
            \STATE Image search $\{I_j^i\}_{j=1}^Q  \leftarrow \texttt{SE}(d_i + c_i, Q)$ \hfill (\S\ref{subsec:text_to_image_search})
            \STATE Calc.\ reward $r_{c_i} \gets \frac 1 Q \sum_{j=1}^Q r(f, \mathcal D, I_j^i)$ \hfill (\S\ref{subsec:image_rel_reward})
        \ENDFOR
        \STATE $\mathcal B_{\text{new}} = \{I_j^1\}_{j=1}^Q \cup \dots \cup \{I_j^M\}_{j=1}^Q$
        \STATE SSL training: $\mathbb{A}(f, \mathcal D \cup \mathcal B \cup \mathcal B_{\text{new}})$ \hfill (\S\ref{subsec:ssl})
        \STATE Add to buffer: $\mathcal{B} \leftarrow \mathcal{B} \cup \texttt{Top50\%}(\mathcal B_{\text{new}}, r)$  %
        \STATE Predict all concept rewards $\mathbf{r}_{\text{concept}}$ from $\{r_{c_i}\}$ \hfill (\S\ref{subsec:unseen_reward})
        \STATE Update concept dist $p \leftarrow \texttt{CalcProb}(\mathbf{r}_{\text{concept}})$ \hfill (\S\ref{subsec:tiering})
    \ENDFOR
\end{algorithmic}
\end{algorithm}

\subsection{Estimating Reward for Unseen Concepts}
\label{subsec:unseen_reward}
Since our vocabulary contains hundreds of thousands of concepts, it is inefficient to search to test whether a query yields relevant images. Luckily, we can estimate the quality of a query by using the observed rewards of the queries used so far. Humans can do this effortlessly due to our understanding of what each concept means. To us, it is obvious that if querying ``golden retriever'' yielded useful images for this dataset, then ``labrador retriever'' probably should as well. To give our method the same understanding of concept meaning, we embed our $146{,}347$ WordNet concepts into a 384-dimensional space using a pre-trained sentence similarity model \cite{reimers2019sentence}. We provide relevant context about concepts to the text embedding model using the following template:
\begin{quote}
\vspace{-0.08in}
{\tt {\small \{lemma\} (\{hypernym\}): \{definition\}}}.
\vspace{-0.08in}
\end{quote}
For example,
\begin{quote}
\vspace{-0.08in}
{\tt {\small Chihuahua (toy dog): an old breed of tiny short-haired dog with protruding eyes from Mexico held to antedate Aztec civilization.}}
\vspace{-0.08in}
\end{quote}

We use Gaussian process regression (GPR) \cite{williams1995gaussian} over the text embeddings $\{\mathbf{e}_i\}$ to predict the concept-level reward $r(\mathbf{e}_i)$ for untried concepts. 
GPR models the function outputs for any set of inputs $\{r(\mathbf{e}_i)\}$ as jointly Gaussian random variables. 
The covariance of any two variables $r(\mathbf{e}_i)$ and $r(\mathbf{e}_j)$ is determined by the kernel $k(\mathbf{e}_i, \mathbf{e}_j)$, which we set as the default RBF kernel $k(\mathbf{e}_i, \mathbf{e}_j) = \exp(\frac{-\|\mathbf{e}_i - \mathbf{e}_j\|_2}{2})$. 
Given the observed rewards for concepts $R_{obs} = \{r(\mathbf e_i)\}$, GPR calculates the posterior distribution over the rewards for an unobserved concept $\mathbf e'$, $P(r(\mathbf e') | \{r(\mathbf{e}_i)\} = R_{obs})$. Given that the joint distribution  $P(\{r(\mathbf{e}_i)\}, r(\mathbf{e}'))$ is Gaussian, the posterior is also Gaussian with mean $\mu(\mathbf e')$ and variance $\sigma(\mathbf e')^2$. The locality provided by the RBF kernel enables reasonable reward predictions, and having a distribution over rewards instead of a point estimate allows us to explore potentially good concepts. We encourage exploration by setting the score of unobserved concepts to $\mu(\mathbf{e}_i) + \sigma(\mathbf{e}_i)$.

\begin{figure}[t]
    \centering
    \includegraphics[width=0.9\linewidth]{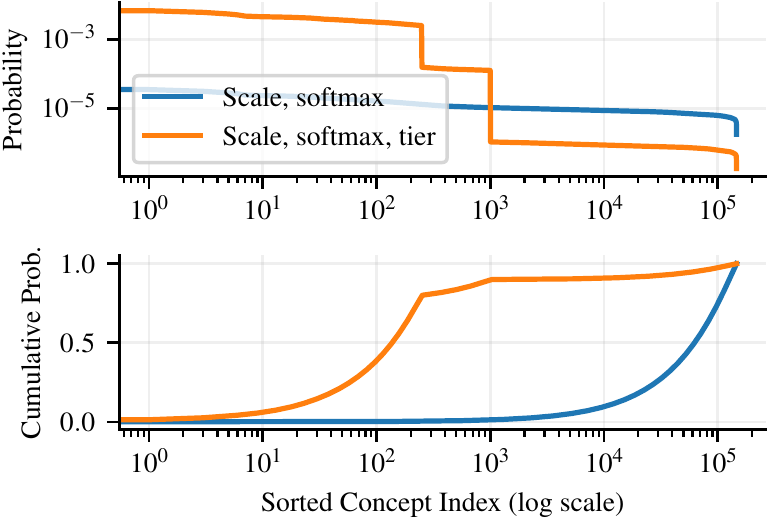}
    \vspace{-0.12in}
    \caption{\textbf{Learned concept sampling distribution.} Given estimated scores for each of the $146,347$ concepts, we need to choose how often to sample each one in order to balance exploration and exploitation. \textbf{Top:} we scale our scores to a desired temperature, then take the softmax to obtain a distribution over concepts. Finally, we create tiers so that the top 250 concepts have $80\%$ of the probability mass, and the next 750 have $10\%$. This ensures that we sample enough from the top $1{,}000$ concepts while still exploring other concepts with lower scores. \textbf{Bottom:} the top $1{,}000$ concepts are only sampled a tiny fraction of the time without tiering.}
    \label{fig:sampling_dist}
    \vspace{-0.12in}
\end{figure}

\begin{figure*}[t]
\centering
\includegraphics[width=0.9\linewidth]{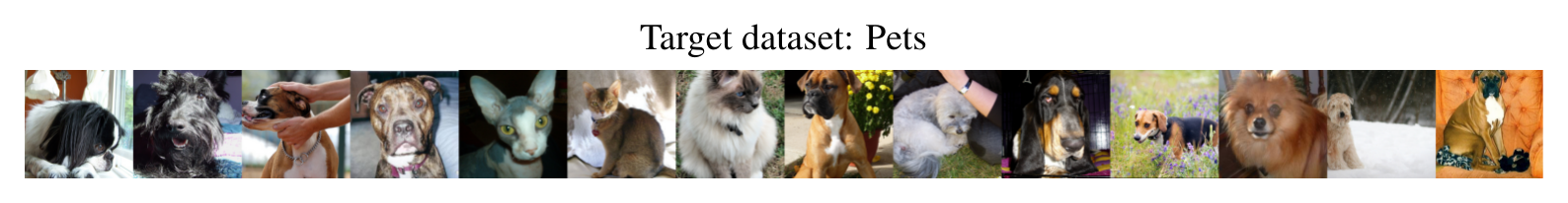}\\
\includegraphics[width=0.9\linewidth]{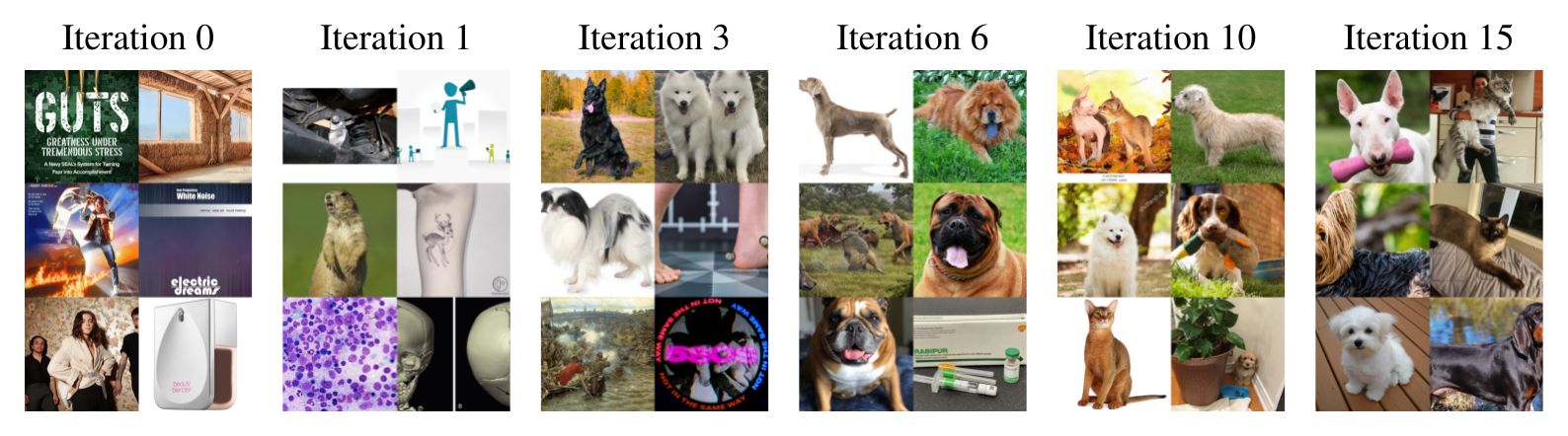}
\vspace{-0.15in}
\caption{
\textbf{Progression of downloaded images across training.} \textbf{Top:} samples of Oxford-IIIT Pets images. \textbf{Bottom:} samples of images queried by Internet Explorer across iterations. As it learns, it makes queries that are progressively more relevant to the target dataset.
}
\label{fig:progression}
\vspace{-0.05in}
\end{figure*}

\subsection{Provable speedup in relevant query identification}
\label{subsec:provable_speedup}
Only a small subset of our vocabulary of $n$ concepts is relevant to the target dataset. We assume that the relevant concepts are partitioned into $c$ disjoint clusters of size $s$, with $cs \ll n $. We want to discover every relevant concept by sampling concepts uniformly at random (with replacement) to test. We assume that sampling a concept conclusively tells us whether it is relevant. Furthermore, we assume that we could optionally use an algorithm (e.g., Gaussian process regression) that, if we have sampled a relevant concept, tells us that all concepts in its cluster are also relevant. Then, Lemma~\ref{lemma:speedup} shows that the Gaussian process drastically reduces the time required to identify all relevant concepts. 
\begin{restatable}{lemma}{lemmaspeedup}
    \label{lemma:speedup}
    Let $T_{base}$ be the expected time to identify every relevant concept without the GPR, and $T_{GPR}$ be the expected time when exploiting the additional knowledge from the GPR. Then, $T_{base} = n H_{c \cdot s}$, $T_{GPR} = \frac{nH_{c}}{s}$, and the speedup from GPR is $\frac{T_{base}}{T_{GPR}} \approx s \log s$.
\end{restatable}
The proof is in Appendix~\ref{sec:proof}. For our vocabulary and target datasets, $s \approx 100$. This shows that a predictive model like GPR is crucial for quickly identifying all useful concepts. 

\subsection{Query sampling distribution}
\label{subsec:tiering}
Once we have estimates for the quality of each concept, how do we determine what to search for next?
We face the age-old dilemma of exploration versus exploitation:
we need to sample the top concepts frequently enough to get relevant training data for SSL, while at the same time, we need sufficient exploration of promising untried concepts.

We use a sampling-based approach based on Boltzmann exploration \cite{sutton1991dyna}. Boltzmann exploration samples based on a scaled softmax distribution $ p(c_i) \propto \exp(r(c_i)/\tau)$, where  
$\tau$ is the temperature scaling.
However, with a large vocabulary (action space) of $146,347$ concepts, it becomes difficult to tune $\tau$ so that we sample the top concepts frequently enough without being too skewed. 
Thus, we define a ``tiering function'' to adjust the probability mass in specified intervals of our distribution. Given a sorted discrete probability distribution $p$, interval boundaries $T_0 =0 < T_1 < \dots < T_n$, and interval masses $\Delta_0, \dots, \Delta_{n-1}$ such that $\sum_i \Delta_i = 1$,  tiering computes a new distribution: 
\begin{align}
    p_i^{\text{tier}} = \Delta_j \frac{p_i}{\sum_{k=T_j}^{T_{j+1}} p_k} \;\;\; \text{for } j \text{  s.t.  }T_j \leq i < T_{j+1} 
\end{align}
$p^{\text{tier}}$ is a new distribution such that $\sum_{k=T_j}^{T_{j+1}} p^{\text{tier}} = \Delta_j$. We use $T_0=0$, $T_1=250$, $T_2=1{,}000$, $T_3=146{,}347$, $\Delta_0=0.8$, $\Delta_1 = 0.1$, and $\Delta_2=0.1$.
Simply put: we give the highest-ranked $250$ concepts $80\%$ of the probability mass, the next $750$ concepts $10\%$, and all remaining concepts $10\%$.
Figure~\ref{fig:sampling_dist} shows that tiering the scaled softmax distribution samples frequently enough from the top concepts while a vanilla scaled softmax distribution does not. 

\section{Experimental Setting}
\subsection{Self-supervised Exploration}
We assume that we have an unlabeled target dataset of images for which we would like to learn useful visual features. We compare three methods:
\begin{enumerate}[noitemsep,topsep=0pt]
    \item Random: sample concepts uniformly from the vocab. 
    \item Ours: sample concepts from our learned distribution. 
    \item Ours++: additionally use GPT-generated descriptors.
\end{enumerate}

\begin{figure*}[t]
    \centering
    \includegraphics[width=\linewidth]{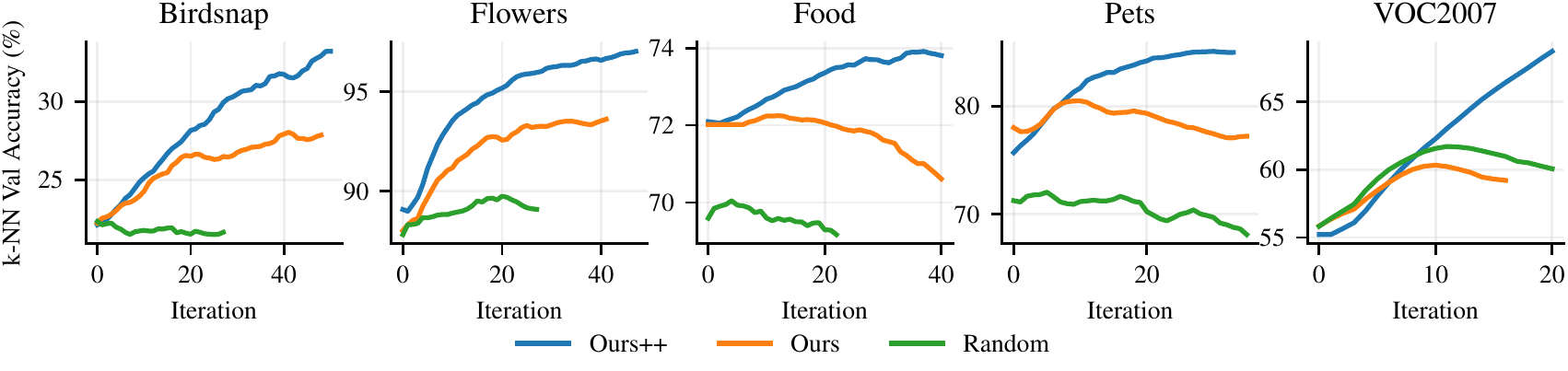}
    \vspace{-2.4em}
    \caption{\textbf{Learning curves in self-supervised setting.} We show how $k$-NN validation accuracy improves across iterations on each target dataset. Without using any labels, Internet Explorer identifies and focuses on relevant concepts for each target dataset. This allows it to find more useful data than the baseline that searches for random concepts. Adding GPT-generated descriptors (Ours++) further improves performance by enabling Internet Explorer to generate diverse views of useful concepts. 
    } 
    \label{fig:learning_curves}
    \vspace{-0.15in}
\end{figure*}

\subsection{Label Set-guided Exploration}
We may sometimes know the set of labels for our task (e.g., ``golden retriever,'' etc.) even if we do not have image-label pairs. 
Knowing the label set greatly accelerates learning on the Internet, because it acts as a strong prior on what could be useful. 
Using our text similarity model, we reduce the size of the vocabulary by selecting the top 
$10\%$ ($14{,}635$ concepts)
with the largest average top-$k$ similarity to the label set in text embedding space. We set $k$ to a third of the size of the label set to reduce the impact of outliers. Reducing the size of the vocabulary strengthens our baselines by ensuring that they only search for potentially useful concepts. We compare 4 methods:   
\begin{enumerate}[noitemsep,topsep=0pt]
    \item Labels: only search for labels. 
    \item Labels + relevant: search for labels
    half of the time, and random concepts from the pruned vocabulary the other half of the time. 
    \item Ours: sample labels half of the time and sample from our learned concept distribution the other half. 
    \item Ours++: additionally use GPT-generated descriptors.
\end{enumerate}
We call this setting ``label set-guided,'' since we have additional supervision in the form of the label set.

\subsection{Datasets and Metrics}
We evaluate Internet Explorer on 4 popular small-scale fine-grained classification datasets: Birdsnap \cite{berg2014birdsnap}, Flowers-102 \cite{nilsback2008automated}, Food101 \cite{bossard2014food}, and Oxford-IIT Pets \cite{parkhi2012cats}.
These small datasets consist of $2{,}040$ to $75{,}750$ training examples, making them ideal for testing whether Internet Explorer can efficiently find relevant useful data.
We also evaluate on PASCAL VOC 2007 (Cls)~\cite{everingham2010pascal}, a coarse-grained multi-label classification task,
and ImageNet-100~\cite{tian2020contrastive}.
Finally, we try FMoW~\cite{fmow2018}, a satellite domain classification task.
We compare the representation quality of our model \wrt its target dataset using two metrics: $k$-nearest neighbors ($k$-NN) accuracy and linear probe accuracy.

\begin{table*}[t]
    \centering
    \begin{adjustbox}{width=1\textwidth}
    \begin{tabular}{lc@{\hskip 0.12em}cc@{\hskip 0.12em}cc@{\hskip 0.12em}cc@{\hskip 0.12em}cc@{\hskip 0.12em}cc@{\hskip 0.12em}cc@{\hskip 0.12em}cc@{\hskip 0.12em}cc}
    \toprule
        Model & \multicolumn{2}{l}{Birdsnap} & \multicolumn{2}{l}{Flowers} & \multicolumn{2}{l}{Food}  & \multicolumn{2}{l}{Pets} & \multicolumn{2}{l}{VOC2007} & \multicolumn{2}{l}{IN100} & \multicolumn{2}{l}{FMoW${}^\star$} &  Images & GPU hrs. \\
    \midrule
    \textit{Fixed dataset, lang. supervision} \\
        \;\;\;CLIP ResNet-50 (\textbf{oracle})  & $57.1$ & & $96.0$ & & $\bf{86.4}$ & & $88.4$ & & $\bf{86.7}$ & & $\bf{89.3}$ & & $44.9$ & & $400 \times 10^6$ & $4{,}000$ \\ %
    \midrule
    \textit{Fixed dataset, self-supervised} \\
        \;\;\;MoCo-v3 (ImageNet pre-train)  & $26.8$ & & $83.2$ & & $70.5$ & & $79.6$ & & $-$ & & $-$ & & $40.8$ & & $1.2 \times 10^6$ & $72$ \\
        \;\;\;MoCo-v3 (ImageNet + target)  & $39.9$ & & $94.6$ & & $78.3$ & & $85.3$ & & $58.0^\dag$ & & $84.7^\dag$ & & $52.5$ & & $1.2 \times 10^6$ & $72 + 12$ \\
    \midrule
    \textit{No label set information} \\
        \;\;\;Random exploration  & $39.6$ & \red{$(-0.3)$} & $95.3$ & \green{$(+0.7)$} & $77.0$ & \red{$(-1.3)$} &  $85.6$ & \green{$(+0.3)$} & $70.2$ & \green{$(+12.2)$} & $85.7$ &  \green{$(+1.0)$} & $54.3$ & \green{$(+1.8)$} &  $2.2 \times 10^6$ & $84 + 40$ \\
        \;\;\;Ours  & $43.4$ & \green{$(+3.5)$} & $97.1$ & \green{$(+2.5)$} & $80.5$ & \green{$(+2.2)$} & $86.8$ & \green{$(+1.5)$} & $68.5$ & \green{$(+10.5)$}  & $86.2$ &  \green{$(+1.5)$} & $-$ & $-$ & $2.2 \times 10^6$ & $84 + 40$ \\
        \;\;\;Ours++  & $54.4$ & \green{$(+14.5)$} & $98.4$ & \green{$(+3.8)$} & $82.2$ & \green{$(+3.9)$} & $89.6$ & \green{$(+4.3)$} & ${80.1}$ & \green{${(+22.1)}$} & $86.4$ &  \green{$(+1.7)$} & $54.1$ & \green{$(+1.6)$} & $2.2 \times 10^6$ & $84 + 40$ \\
    \midrule 
    \textit{Use label set information} \\       
        \;\;\;Search labels only  & $47.1$ & \green{$(+7.2)$} & $96.3$ & \green{$(+1.7)$} & $80.9$ & \green{$(+2.6)$} & $85.7$ & \green{$(+0.4)$} & $61.8$ & \green{$(+3.8)$} & $85.7$ &  \green{$(+1.0)$} & $53.5$ & \green{$(+1.0)$} & $2.2 \times 10^6$ & $84 + 40$ \\
        \;\;\;Labels + relevant terms  & $49.9$ & \green{$(+10.0)$}& $98.0$ & \green{$(+3.4)$} & $81.2$ & \green{$(+2.9)$} & $87.0$ & \green{$(+1.7)$} & $67.5$ & \green{$(+9.5)$} & $86.3$ &  \green{$(+1.6)$} & $54.1$ & \green{$(+1.6)$} & $2.2 \times 10^6$ & $84 + 40$ \\
        \;\;\;Ours  & $52.0$ & \green{$(+12.1)$} & $97.6$ & \green{$(+3.0)$} & $81.2$ & \green{$(+2.9)$} & $87.3$ & \green{$(+2.0)$} & $70.3$ & \green{$(+14.3)$} & 86.4 &  \green{$(+1.7)$} & -- & -- & $2.2 \times 10^6$ & $84 + 40$ \\
        \;\;\;Ours++  & $\mathbf{62.8}$ & \green{$\mathbf{(+22.9)}$} & $\bf{99.1}$ & \green{$\mathbf{(+4.5)}$} & $84.6$ & \green{$(+6.3)$} & $\mathbf{90.8}$ & \green{$\mathbf{(+5.5)}$} & ${79.6}$ & \green{$(+21.6)$} & $87.1$ &  \green{$(+2.4)$} & $\mathbf{54.5}$ & \green{$(+\mathbf{2.0})$} & $2.2 \times 10^6$ & $84 + 40$ \\
    \bottomrule
    \end{tabular}
    \end{adjustbox}
    \caption{\textbf{Linear probing accuracy}. Our method significantly improves the starting checkpoint performance in just 40 additional hours of training. We show the performance change from the starting MoCo-v3 (ImageNet + target) initialization in \green{green}/\red{red}. CLIP numbers correspond to linear probe (which is higher than its zero-shot accuracy). Internet Explorer reaches or often surpasses CLIP (oracle with 2x params) performance on each dataset while using 2.5\% as much compute and 0.5\% as much data.
    ${}^{\dag}$For VOC2007 and IN100, we do not do ImageNet pre-training because ImageNet is too similar and obscures the effect.
    ${}^\star$For FMoW-WILDS, we use a hand-crafted list of domain-specific descriptors common to all models
    (see
    \cref{sec:fmow_wilds_details}
    for more details).
    }
    \label{tab:main_results}
    \vspace{-0.15in}
\end{table*}

\begin{figure}[t]
\centering
\includegraphics[width=\linewidth]{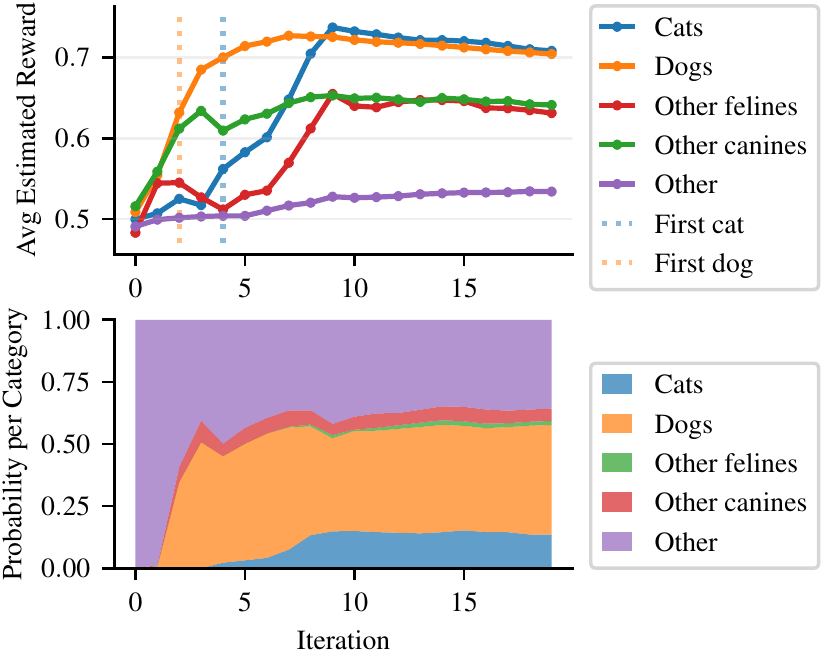}
\vspace{-0.3in}
\caption{\textbf{Self-supervised concept discovery on Pets dataset.} When targeting the Pets dataset, self-supervised Internet Explorer quickly estimates high reward for concepts from the cat category (82 concepts) and dog category (246 concepts). It is also able to identify felines that are not cats (e.g., tiger) and canines that are not dogs (e.g., wolf), although it gives them lower reward on average. Finding these categories is especially challenging since they comprise only $460/146{,}347 = 0.3\%$ of the vocabulary.}
\vspace{-0.22in}
\label{fig:reward_over_training}
\end{figure}

\section{Results and Analysis}
\subsection{Self-supervised Results}
\cref{fig:learning_curves} shows how Internet Explorer improves the $k$-NN accuracy more efficiently than sampling queries uniformly at random from the concept vocabulary. In fact, random sampling occasionally decreases accuracy, likely due to the fact that Internet images can generally be unsuitable for pre-training due to issues such as watermarks, images containing text, and overly photogenic images \cite{mezuman2012learning,chen2015webly}. 
\cref{tab:main_results} shows that our method significantly improves on the starting MoCo-v3 (ImageNet + target) checkpoint and can outperform a CLIP \cite{radford2021learning} model of the same size while using much less compute and data. This is impressive as CLIP can be considered an oracle since its training set contains up to 20k Bing image search results for each WordNet lemma (in addition to other queries).
Using GPT-generated descriptors in ``Ours++'' also significantly improves performance by enabling Internet Explorer to generate diverse views of the most useful concepts. 

\subsection{Self-supervised Exploration Behavior}
\cref{fig:reward_over_training} shows the progression of Internet Explorer (Ours++) behavior on the Pets dataset in the self-supervised setting. Since Pets consists of cat and dog breeds, to analyze the results, we use the WordNet hierarchy to divide concepts in our vocabulary into 5 meaningful categories: cats, dogs, non-cat felines (e.g., lion), non-dog canines (e.g., wolf), and other. This categorization is only done for this post hoc analysis and is not provided during training. Figure~\ref{fig:reward_over_training} (top) shows that
Internet Explorer rapidly identifies the roughly $0.3\%$ of concepts that are useful for Pets. During the first two iterations, the average estimated reward for each category is roughly the same. However, after the first dog concept is searched in iteration $\#2$, the estimated reward and probability mass for dogs and other canines rapidly increases. The same happens for cats after the first cat is searched in iteration $\#4$. Interestingly, while ``other felines'' and ``other canines''  have higher average reward than the ``other'' category, they still have much lower reward than cats and dogs. This indicates that our model understands that other felines and canines (mostly large, wild predators) are only moderately relevant for house pet cats and dogs. 

Figure~\ref{fig:progression} shows how Internet Explorer downloads progressively more useful images over time. It shows 8 random images that were downloaded in iteration $\#0$, $\#1$, $\#3$, $\#6$, $\#10$, and $\#15$ in the self-supervised setting. Iteration $\#0$ contains mostly useless data, like graphics or screenshots, but Pets-relevant images already make up most of the downloads by iteration $\#3$. \cref{sec:progression_downloaded_imgs} shows that Internet Explorer identifies useful images shockingly quickly across every dataset, without any knowledge of their label sets.

\begin{figure*}[t]
    \centering
    \includegraphics[width=\linewidth]{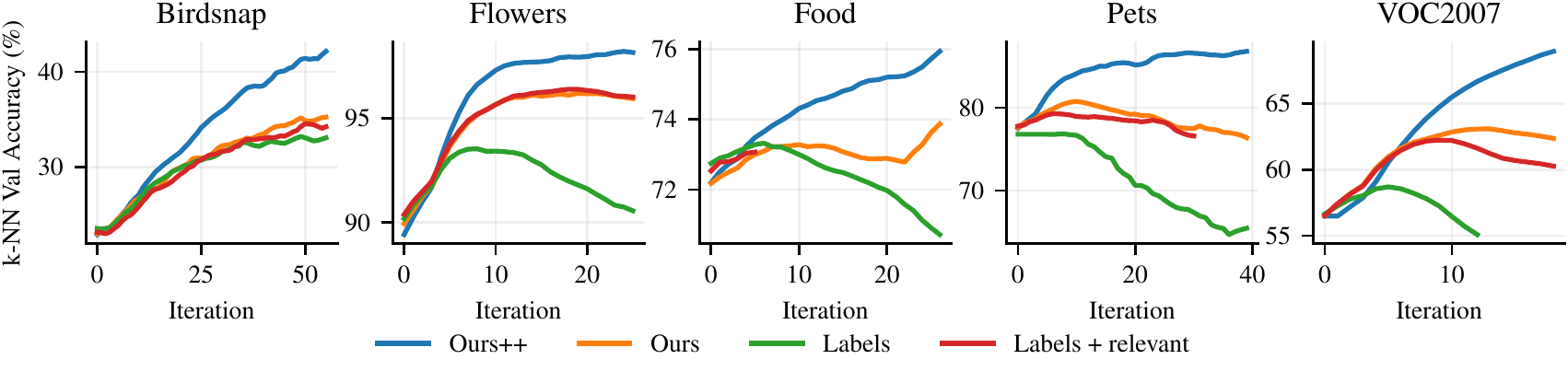}
    \vspace{-0.3in}
    \caption{\textbf{Learning curves in label set-guided setting.} Using knowledge of the label set improves the performance of all methods. }
    \label{fig:semisup_learning_curves}
    \vspace{-0.1in}
\end{figure*}

\subsection{Label Set-guided Results}
Internet Explorer significantly outperforms the stronger baselines in the label set-guided setting where we additionally have knowledge of the label set. Searching for the label set continuously provides useful data and helps us rapidly identify other useful concepts. Together with the diversity promoted by GPT descriptors, Ours++ outperforms CLIP in 4/7 datasets and approaches its performance in the other 3, using just 2.5\% of the time and 0.5\% the data.

\subsection{Domain dataset results}
\label{sec:domain_dataset_results}
To test if Internet Explorer is effective when the target dataset contains very specific domain knowledge, we apply it to
FMoW-WILDS~\cite{fmow2018}---a popular satellite imaging domain dataset---by hand-designing a dozen search prompts that help induce satellite image results (details in \cref{sec:fmow_wilds_details}).
Even though the WordNet vocabulary is not particularly suited for this dataset, Internet Explorer still improves the LP accuracy by $2$ percentage points (see \cref{tab:main_results}).
Notably, all of our methods dramatically outperform CLIP here, likely because the distribution of satellite data is very different than the data used to train CLIP.
This demonstrates the wide flexibility of our method to be applied to arbitrary domains.

\subsection{Learning from other sources of data}
\label{subsec:search_engine_main}
\begin{table*}[t]
    \centering
    \begin{adjustbox}{width=0.85\textwidth}
    \begin{tabular}{@{\extracolsep{4pt}}lccccccccc}
    \toprule
        \multirow{2}{*}{\textbf{Model}}
        &\multicolumn{3}{c}{\textbf{Flowers}} 
        &\multicolumn{3}{c}{\textbf{Food}}
        &\multicolumn{3}{c}{\textbf{Pets}} \\
        \cmidrule{2-4} \cmidrule{5-7} \cmidrule{8-10}

        & Google & Flickr & LAION & Google & Flickr & LAION & Google & Flickr & LAION \\
    \midrule
    \textit{Fixed dataset} &&&&&&&&&\\    
        \;\;\; MoCo-v3 (IN)                          & $83.2$ & $83.2$ & $83.2$ & $70.5$ & $70.5$ & $70.5$ & $79.6$ & $79.6$ & $79.6$ \\
        \;\;\; MoCo-v3 (IN + target)                 & $94.6$ & $94.6$ & $94.6$ & $78.3$ & $78.3$ & $78.3$ & $85.3$ & $85.3$ & $85.3$ \\
    \midrule
    \textit{Undirected search} &&&&&&&&&\\    
        \;\;\;Random exploration                     &  $95.3$  &  $95.2$  &  $94.8$  &  $77.0$ &  $80.0$  &  $80.2$  &  $85.6$ & $84.4$  & $85.1$ \\
    \midrule 
    \textit{Internet Explorer} &&&&&&&&&\\    
        \;\;\;Ours++ (no label set)                  &  $98.4$  &  $98.1$  &  $94.6$  &  $81.2$  &  $80.3$  &  $80.9$  &  $87.3$  &  $88.4$  &  $85.9$  \\
        \;\;\;Ours++ (with label set)                &  $\bf{99.1}$ &  $\bf{99.0}$ &  $\bf{95.8}$ & $\bf{84.6}$ & $\bf{81.9}$  &  $\bf{81.0}$  & $\bf{90.8}$ &  $\bf{89.1}$  &  $\bf{86.7}$  \\
    \bottomrule
    \end{tabular}
    \end{adjustbox}
    \vspace{-0.06in}
    \caption{\textbf{Linear probe accuracy with other search engines}. Internet Explorer improves its performance using any search engine, including Flickr and our custom text-based LAION search engine.}
    \label{tab:search_engine}
    \vspace{-0.05in}
\end{table*}

We primarily obtain images by querying Google Images, but Internet Explorer is compatible with any text-to-image search engine. To measure the effect of the choice of search engine, we also test Internet Explorer with the Flickr photo search API and a custom search engine we built on top of a subset of LAION-5B~\cite{schuhmann2022laion}. LAION-5B consists of noisy web-scraped (text, image) pairs, and our custom LAION search engine searches using approximate nearest neighbors in \textit{text embedding space}. Thus, it tests whether Internet Explorer can still improve even when the search engine has little inductive bias. We discuss more details in \cref{sec:search_engine_details}. \cref{tab:search_engine} shows that Internet Explorer consistently improves over time, regardless of the search engine we use. Google consistently does best, followed by Flickr, then LAION (which has the smallest pool of images to draw from). Using Internet Explorer to search LAION-5B consistently performs \textit{better} than random exploration---indicating that Internet Explorer is effective even for selecting data from a static dataset.

\subsection{Effect of image reward type}
\label{subsec:reward_analysis}

We run an ablation on the type of image relevance reward. Instead of calculating the image reward based on the average similarity to the $k=15$ nearest neighbors in representation space (as in \cref{subsec:ssl}), we also try using $k=1$
\begin{wraptable}{r}{0.5\linewidth}
    \vspace{-0.05in}
    \centering
    \begin{adjustbox}{width=0.9\linewidth}
        \begin{tabular}{lcc}
        \toprule
        Reward Type & Food \\
        \midrule
        MoCo loss & 81.2 \\
        1-NN sim  & 83.2 \\
        15-NN sim (ours) & \textbf{84.6} \\
        \bottomrule
    \end{tabular}
    \end{adjustbox}
    \vspace{-0.1in}
    \caption{\textbf{Ablation on type of image reward.}
    MoCo loss does not identify relevant concepts, and 1-NN is sensitive to outlier images. }
    \label{tab:image_reward}
    \vspace{-0.15in}
\end{wraptable}
or the MoCo contrastive loss as the reward. \cref{tab:image_reward} compares these three metrics in the label set-guided setting and shows that $k=15$ does best. We explain this result by qualitatively comparing the behavior of various metrics on Food101 in \cref{fig:reward_ranking}.
The MoCo loss does not identify relevant concepts, instead preferring images that are difficult to align across augmentations.
Representation similarity with $k=1$ also fails, as it prefers images of \texttt{zebras} and \texttt{books} because they are highly similar to a few outlier images in Food101. Our proposed reward with $k=15$ eliminates the influence of outliers and avoids this problem.

\begin{figure}[t]
    \centering
    \includegraphics[width=0.9\linewidth]{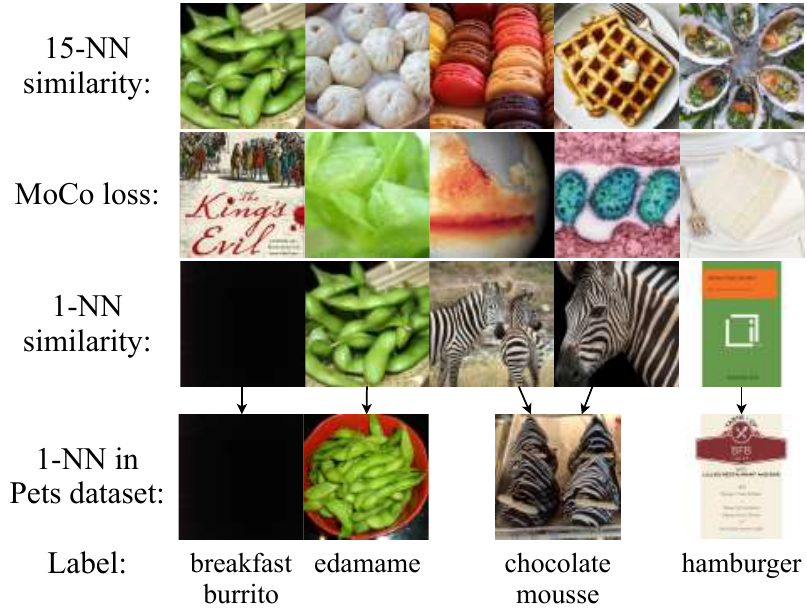}
    \vspace{-0.1in}
    \caption{\textbf{Most preferable images under different rewards.} We show the top 5 downloaded images ranked by 3 possible image rewards 
    for adversarial Food101 examples.
    MoCo loss encourages noisy out-of-distribution images; 15-NN (ours) prefers a wide variety of food images, whereas outliers in the Food dataset throw off 1-NN, causing it to reward black images, text, and zebras.}
    \label{fig:reward_ranking}
    \vspace{-0.15in}
\end{figure}

\subsection{Comparison to image-to-image search}
An alternate approach to finding relevant Internet data is to use image-to-image search: for each image in the target dataset, directly retrieve images that are visually similar. 

\vspace{-0.1in}
\paragraph{Scientific and practical issues}
Image-to-image search uses strong visual representations from pretrained models in order to identify similar images. This defeats the primary purpose of Internet Explorer: learning useful representations when none exist beforehand (e.g., a new iPhone is released that is out-of-distribution for existing vision models). Text-based search avoids this issue by using additional supervision (e.g., caption and surrounding text) that makes it easier to index new images. Image-to-image search also relies on paid APIs that can cost thousands of dollars. 

\vspace{-0.1in}
\paragraph{Comparison to text-based search}
Regardless of the concerns above, we do a controlled comparison between Internet Explorer and image-based search over LAION-5B. For each image in a target training set, we compute its CLIP ViT-L/14 representation and find its $N$ nearest neighbors in LAION-5B. We choose $N$ so that we download a total of 1 million new images, which matches how many images Internet Explorer downloads. We then train a MoCo-v3 model on a 1:1 mix of the target dataset and the downloaded images with the exact same hyperparameters (e.g., learning rate, number of steps, etc) as Internet Explorer. Interestingly, Table~\ref{tab:image-to-image} shows that the image-to-image approach consistently learns worse features than Internet Explorer, despite taking advantage of strong, pretrained vision features from CLIP. We hypothesize that image-to-image search finds images that are too similar to the target images, resulting in less additional information that was not already present in the target dataset. In contrast, using text (concepts and descriptors) as an intermediate bottleneck encourages Internet Explorer to download novel images that generalize along useful axes.

\section{Related Work}

Many papers use self-supervised or weakly-supervised learning on large-scale static datasets collected from the Internet, such as YFCC-100M \cite{thomee2015yfcc100m}, Instagram-1B~\cite{mahajan2018exploring}, or LAION-5B \cite{schuhmann2022laion}. However, these are usually impractically expensive since they train on all of the data, not just the subset relevant to a target dataset. 
Concurrent work~\cite{oquab2023dinov2} attempts to address this by adding a ``one-time'' automatic data curation step that keeps only the most relevant images from a static web crawl dataset. This approach works well but is limited as the selection process does not use the most up-to-date learned features or adjust its searches on-the-fly to focus on especially useful data. 

Other approaches obtain additional training data by searching for predetermined queries. \citet{fergus2005learning} create a supervised training dataset from the Google image search results for a list of known classes. \citet{kamath2022webly} improve a visual question-answering model using a set of predetermined Bing queries. However, these approaches query the internet just once, which is susceptible to noise in the search results, and the total amount of data is limited to the relevant search terms known \textit{a priori}. Internet Explorer's self-supervised approach bypasses these problems. It can learn useful features from noisy yet relevant data, and it only needs an initial image collection to identify relevant search queries. This enables it to \textit{continually} explore the Internet via a potentially unbounded number of searches. 

\begin{table}[t]
    \vspace{-0.12in}
    \centering
        \begin{adjustbox}{width=0.85\linewidth}
        \begin{tabular}{lccc}
        \toprule
            & Flowers & Pets & VOC2007 \\
        \midrule
        Image-to-image & 96.6 & 81.6 & 67.8 \\
        Internet Explorer (ours) & \textbf{98.8} & \textbf{87.0} & \textbf{76.1} \\
    \bottomrule
\end{tabular}
\vspace{-0.5in}
\end{adjustbox}
    \caption{
        \textbf{$k$-NN accuracy across search methods}. Image-to-image search uses CLIP ViT-L/14 vision features to acquire the nearest neighbors of each target dataset image. Despite using strong pretrained features and the same source data (LAION-5B), number of downloaded images, and other hyperparameters as Internet Explorer, the image-to-image approach learns worse features. 
    }
    \label{tab:image-to-image}
    \vspace{-0.22in}
\end{table}

Finally, some approaches continuously interact with the Internet to find useful data.
NELL \cite{carlson2010toward,mitchell2018never} extracts text from web pages to form beliefs, and NEIL \cite{chen2013neil} uses images downloaded from Google Image Search to learn visual concepts.
However, both methods are undirected (\ie, they do not modify their exploration behavior to prioritize specific data), which means that learning is slow and will not necessarily improve performance on a desired task. In contrast, Internet Explorer continually uses \textit{targeted} exploration on the Internet to find data for self-supervised training.

\section{Conclusion}
We show that interactively exploring the Internet is an efficient source of highly relevant training data---if one knows how to search for it. In just 30--40 hours of training on a single GPU, Internet Explorer significantly outperforms or closely matches the performance of compute-heavy \textit{oracle} models like CLIP trained on static datasets, as well as strong baselines that search the Internet in an undirected manner.

\noindent \textbf{Acknowledgements} We thank Russell Mendonca for helpful discussions and Shivam Duggal, Mihir Prabhudesai, Sheng-Yu Wang, Jason Y. Zhang, and Rishi Veerapaneni for paper feedback. AL is supported by the NSF GRFP, grants DGE1745016 and DGE2140739. This work is supported by NSF IIS-2024594 and ONR MURI N00014-22-1-2773.

\bibliography{main}
\bibliographystyle{icml2023}

\newpage
\appendix
\onecolumn
\section*{Appendix}
\section{Learning from other sources of data}
\label{sec:search_engine_details}
\begin{wrapfigure}{R}{0.375\textwidth}
\centering
    \vspace{-1.5em}
    \includegraphics[width=\linewidth]{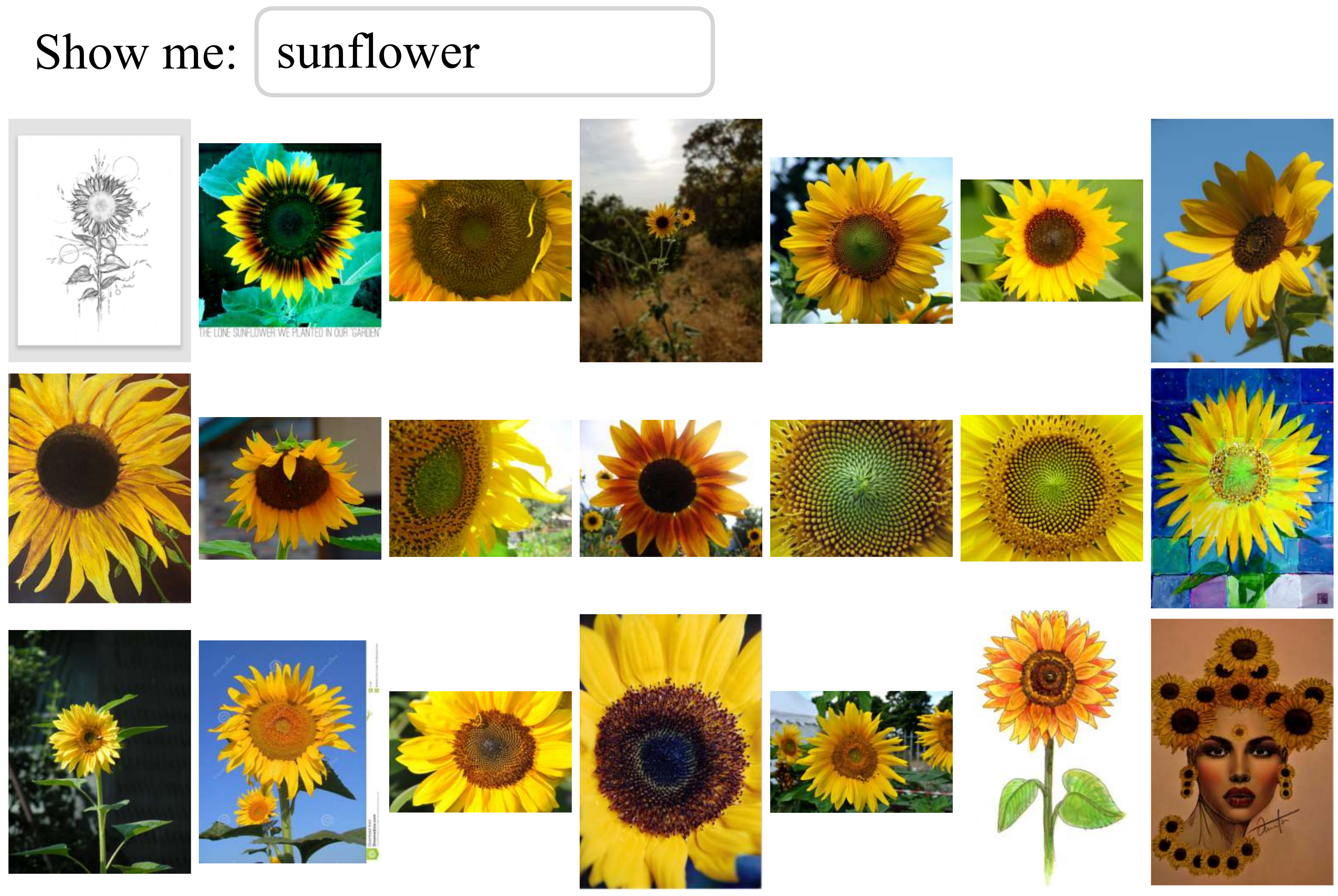}
    \vspace{-1em}
    \caption{\textbf{Our custom LAION-5B search engine.} We build a custom text-to-image search engine that finds images within the LAION-5B dataset by doing nearest neighbor search in text embedding space. This uses no image features whatsoever.}
    \label{fig:laion_search_engine}
    \vspace{-0.75em}
\end{wrapfigure}

\begin{figure*}[b]
    \centering
    \includegraphics{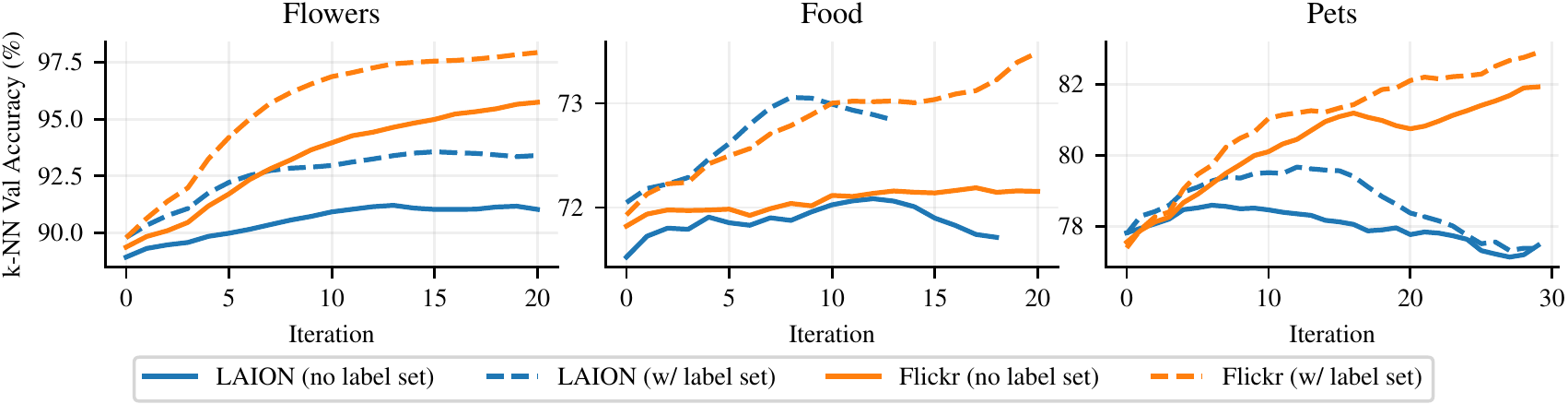}
    \caption{\textbf{Learning from Flickr and LAION-5B.} Even with the noisy search results returned by Flickr and LAION, Internet Explorer still continuously improves performance. }
    \label{fig:other_data_curves}
\end{figure*}

Google Images is an exceptionally useful data source for Internet Explorer. It offers access to a large portion of the Internet's images, and it ranks images using weak supervision from the image caption, surrounding text, click rates, image features, incoming and outgoing hyperlinks, and other signals. This extra supervision is helpful and should be utilized. Nonetheless, we show that Internet Explorer is agnostic to the choice of text-to-image search engine and can still rapidly improve even when the data source is much noisier. 

To test Internet Explorer in the most minimal setting, we build a custom search engine that finds images solely using their accompanying text---without using any pre-trained visual features whatsoever. We use the LAION-5B dataset~\cite{schuhmann2022laion}, which consists of 
\textgreater 5B
noisy image-caption pairs. We filter the dataset to only include 
images of at least $512^2$ pixels
with English captions.
This leaves us with about 600M text-image pairs. To find image results for a query, we find the 100 captions closest to the query in text representation space, then return the associated images.
We use a pre-trained text embedding model~\cite{reimers2019sentence} to compute 384-dimensional text embeddings for each caption. Then, we use Faiss~\cite{johnson2019billion} to compute a fast, approximate nearest-neighbors lookup index. Querying our custom search engine finds 100 image results in less than a second. \cref{fig:laion_search_engine} shows that our search engine is reasonably accurate, even without using any image features. 

We also test Flickr's photo search API as another text-to-image search engine, in addition to Google Images and LAION. \cref{fig:data_comparison} shows that each data source has its own tendencies. For the ``spaghetti bolognese'' query, Google Images is biased~\cite{mezuman2012learning,chen2015webly} towards brightly-lit, photogenic images that typically come from food blogs. Flickr mainly consists of amateur home photos, so it returns a messier variety of images that perhaps better capture the real world. LAION images come from web crawling, without any ranking, so they additionally contain many graphics with text overlays. The same image can also frequently show up in the LAION results multiple times, as a result of being posted on multiple separate pages.

\begin{figure*}
    \centering
    \includegraphics{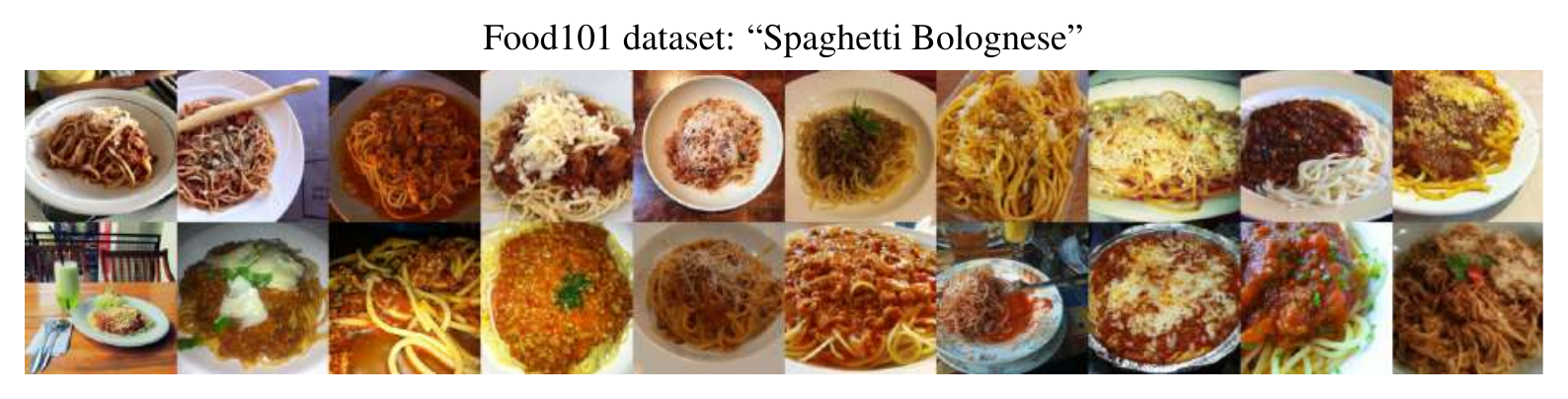} \\
    \includegraphics{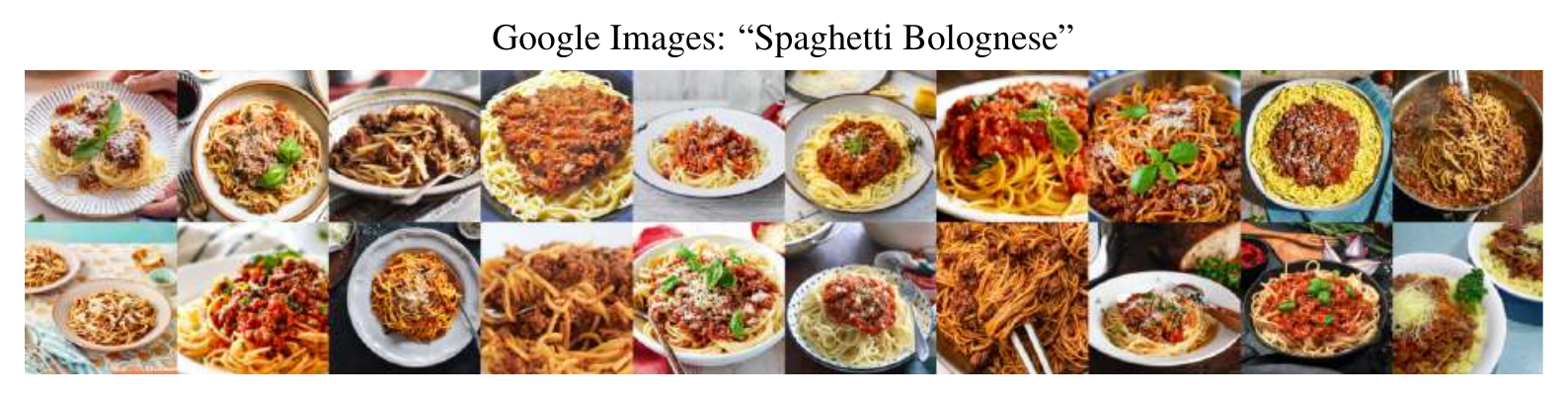} \\
    \includegraphics{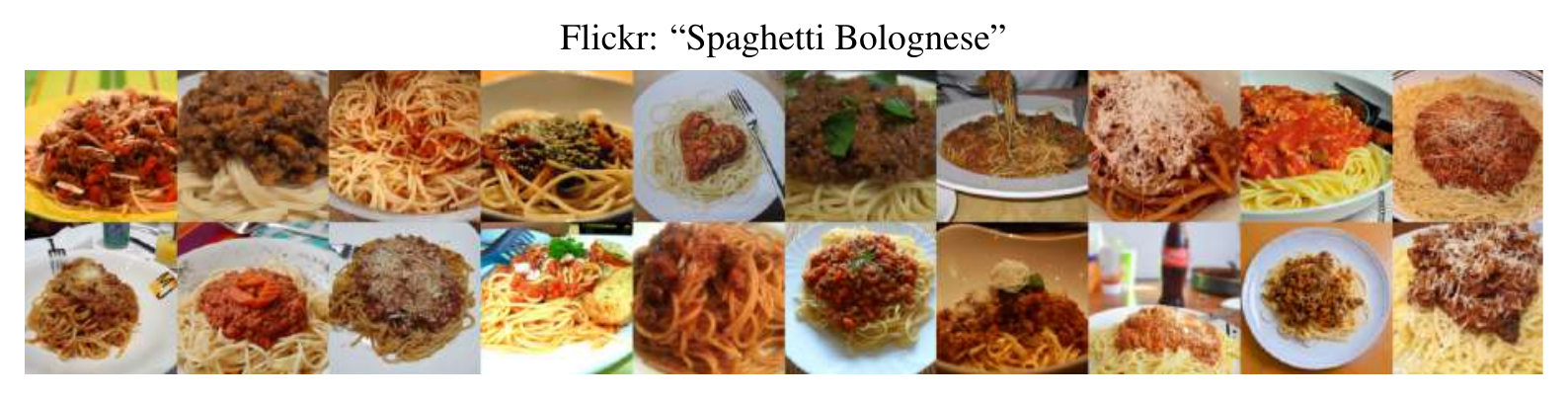} \\
    \includegraphics{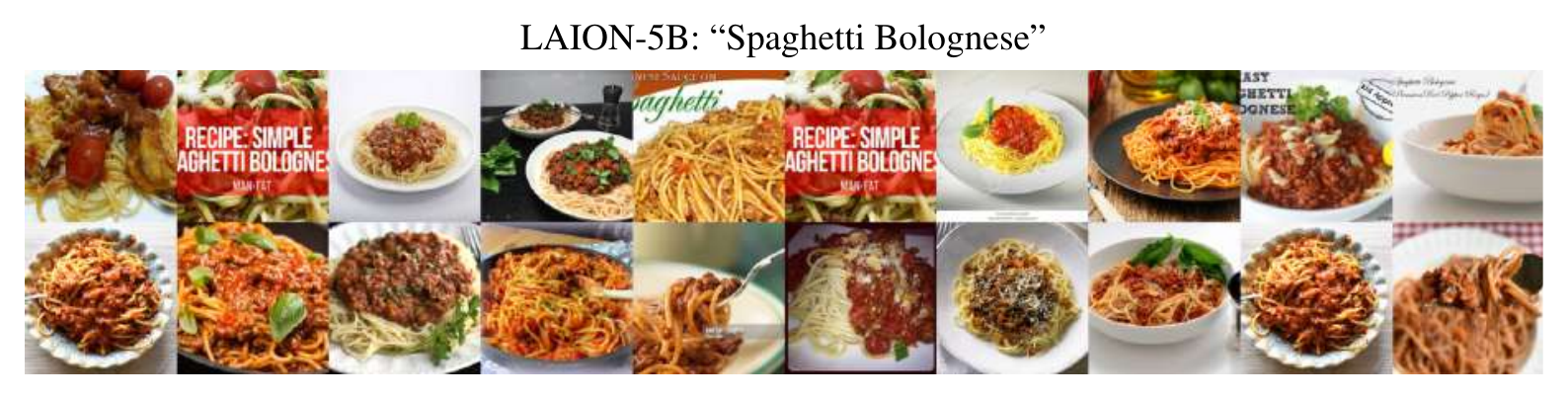} 
    \caption{\textbf{Comparison of different search engines.} We show images for the ``spaghetti bolognese'' class in the Food101 dataset, as well as 20 search results for ``spaghetti bolognese'' from Google Images, Flickr, and LAION5B. Google images are typically well-lit, aesthetic food blog pictures. In comparison, Flickr images are messier, darker, and capture a wider variety of real-world conditions. LAION-5B images lie somewhere in the middle, but contain text overlays much more frequently. Duplicate image results are also common.}
    \label{fig:data_comparison}
\end{figure*}

\cref{fig:other_data_curves} and \cref{tab:search_engine} (main paper) show that Internet Explorer still improves over time, even when the data comes from LAION or Flickr. 
Internet Explorer tends to perform better with Flickr than with LAION, which makes sense. Flickr indexes far more images, as our custom LAION search engine only uses 600M images, so it can return more of the useful photos that Internet Explorer queries for. Flickr is also slightly better at understanding descriptors, although both Flickr and LAION tend to be thrown off by specific or odd descriptors. Nevertheless, Internet Explorer significantly improves the starting model in less than a day of searching and training even with noisy search results and no hyperparameter tuning. Overall, these results prove that Internet Explorer can effectively utilize any window into the Internet's vast ocean of image data.

\section{Are we finding the entire test set online?}
\label{sec:finding_test_set_online}

\newcommand{\blue}[1]{\textcolor{Cerulean}{#1}}
\begin{table}[t]
    \centering
        \begin{tabular}{
            lr@{\hskip 0.12em}
            rr@{\hskip 0.12em}
            rr@{\hskip 0.12em}
            rr@{\hskip 0.12em}
            rr@{\hskip 0.12em}
            rc}
        \toprule
            &
            \multicolumn{2}{l}{Birdsnap} & 
            \multicolumn{2}{l}{Flowers} & 
            \multicolumn{2}{l}{Food} & 
            \multicolumn{2}{l}{Pets} & 
            \multicolumn{2}{l}{VOC2007} \\
        \midrule
        Target test set size                          &  \multicolumn{2}{l}{$1849$} &  \multicolumn{2}{l}{$6142$} & \multicolumn{2}{l}{$25246$} &\multicolumn{2}{l}{$3663$} &\multicolumn{2}{l}{$4952$} \\ 
        \midrule
        \textit{No exploration} \\
            \;\;\; Target training set overlap                          &  $1$ & \blue{$(0.05\%)$} &  $5$ & \blue{$(0.01\%)$}& $34$ & \blue{$(0.13\%)$} & $21$ & \blue{$(0.57\%)$} &  $0$ & \blue{$(0.00\%)$} \\
        \midrule
        \textit{Internet Explorer} \\
            \;\;\;Ours++ (no label set)                         &  $28$ & \ \blue{$(+1.46\%)$} & $11$ & \ \blue{$(+0.01\%)$} & $35$ & \ \blue{$(+0.00\%)$} & $26$ & \ \blue{$(+0.14\%)$}& $1$ & \ \blue{$(+0.02\%)$} \\
            \;\;\;Ours++ (with label set)                       & $57$ & \blue{$(+3.03\%)$} & $27$ & \blue{$(+0.36\%)$}& $35$ &\blue{$(+0.00\%)$} & $43$ &\blue{$(+0.60\%)$} & $1$ &\blue{$(+0.02 \%)$} \\
    \bottomrule
\end{tabular}
    \caption{
        \textbf{Number of leaked test set images}. We use image hashing to compute the fraction of test images present in the set of images downloaded by Internet Explorer. 
        Surprisingly, the training/validation sets of these datasets already leak a small fraction of the test sets---Pets is the most egregious, with $0.57\%$ 
        test leakage.
        For each dataset, we show the test set size, the number of leaked test images, and the percentage of the test set that this represents in \blue{blue}. For each version of our method, we show the total number of leaked images that the model had access to, and the percentage increase this represents over the training set's leakage in \blue{blue}.
        Leakage numbers for our methods include this train-test leakage, since our methods also train on the target dataset's training set. Internet Explorer only finds a tiny fraction of test set images online, and it only uses them for self-supervised training, so there is no \textit{label leakage}. 
        Internet Explorer's large increase in accuracy cannot be explained by test set leakage,
        so its performance gains must come through better feature learning and generalization.
    }
    \label{tab:leakage}
\end{table}

One may be concerned that Internet Explorer improves performance mainly by finding a significant portion of the test set images online. We address this concern by checking how much test data Internet Explorer has downloaded. We use difference hashing (dHash)~\cite{imagehash} to compute hashes for the target dataset's training set, its test set, and the $\approx 10^6$ images that Internet Explorer has downloaded. We compare hashes to determine how many test images were leaked, and we report the number of collisions in \cref{tab:leakage}. Across all five datasets, Internet Explorer finds very few test images. On Birdsnap, Internet Explorer finds 56 additional test set images that were not leaked in the training set, which is roughly $3\%$ of the test set. On the other datasets, the amount leaked ranges from $0.003\%$ to $0.6\%$ of the test set. Additionally, we only perform image-based self-supervised training on downloaded images, so it is much harder for our model to cheat with the leaked images. Overall, given that Internet Explorer outperforms its starting checkpoint by between 5 to 30 percentage points, we conclude that its performance cannot be explained by cheating.

In fact, we view it as a positive that Internet Explorer finds some test set images, because it serves as confirmation that it is learning to search for relevant images---and the most relevant images possible would be those from the dataset itself!
But beyond test set images, Internet Explorer finds a lot of internet images that are very relevant to the dataset. We visualize the top-10 most similar downloaded images for 5 randomly selected test set images from multiple datasets in \cref{fig:internet-nns:pets,fig:internet-nns:food,fig:internet-nns:flowers,fig:internet-nns:voc,fig:internet-nns:in100}.
We use CLIP ViT-L/14 to compute the representations of the test set images, as well as the downloaded images. We then find the top-10 most similar online images given a test set image (from the downloaded images using Ours++ (with label set)).
We see that Internet Explorer finds several images that are very similar but not identical to the test set images.

\begin{figure}
    \centering
    \includegraphics[width=0.9\linewidth]{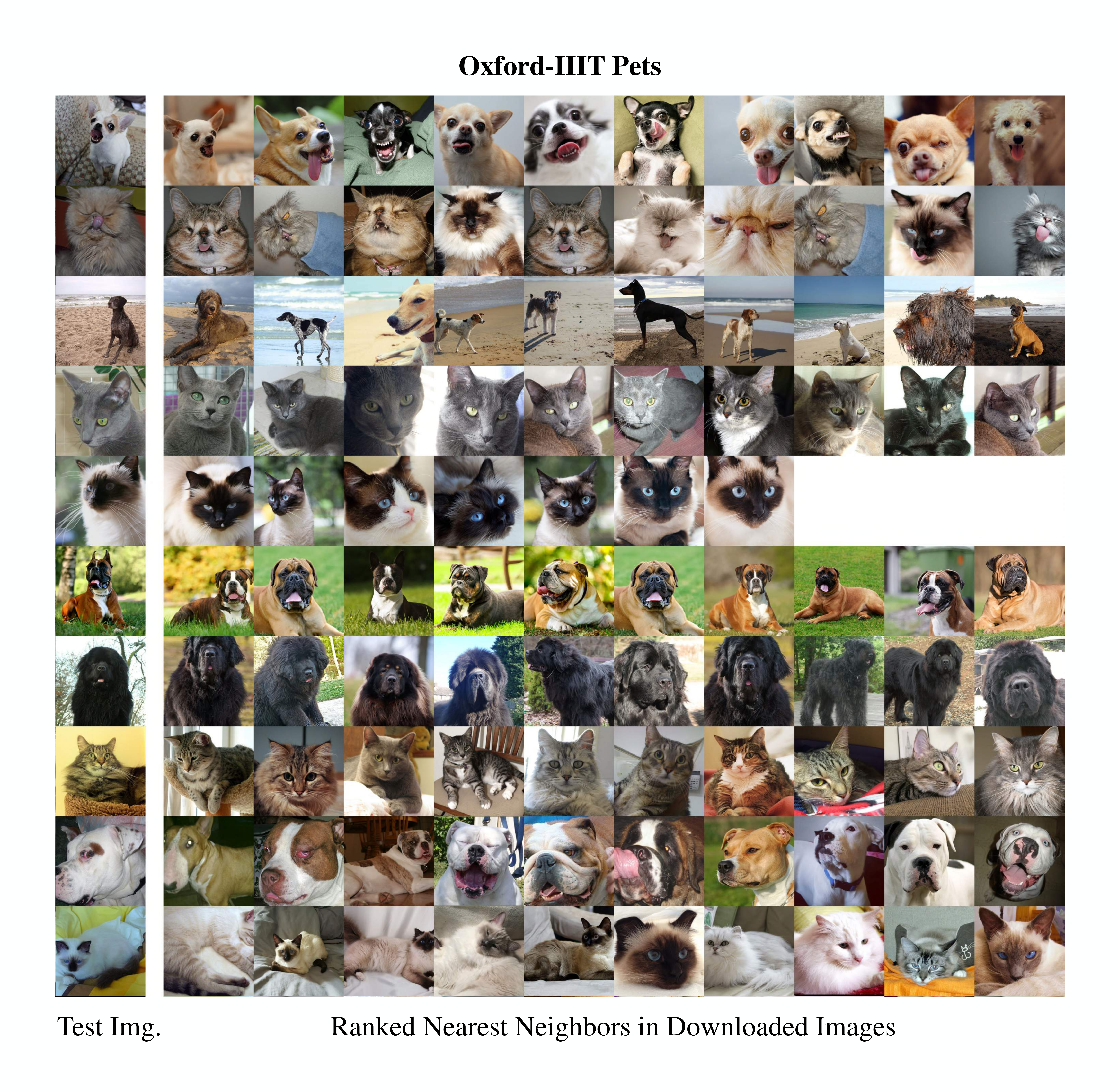}
    \caption{\textbf{Top-10 most similar online images to Pets101}}
    \label{fig:internet-nns:pets}
\end{figure}

\begin{figure}
    \centering
    \includegraphics[width=0.9\linewidth]{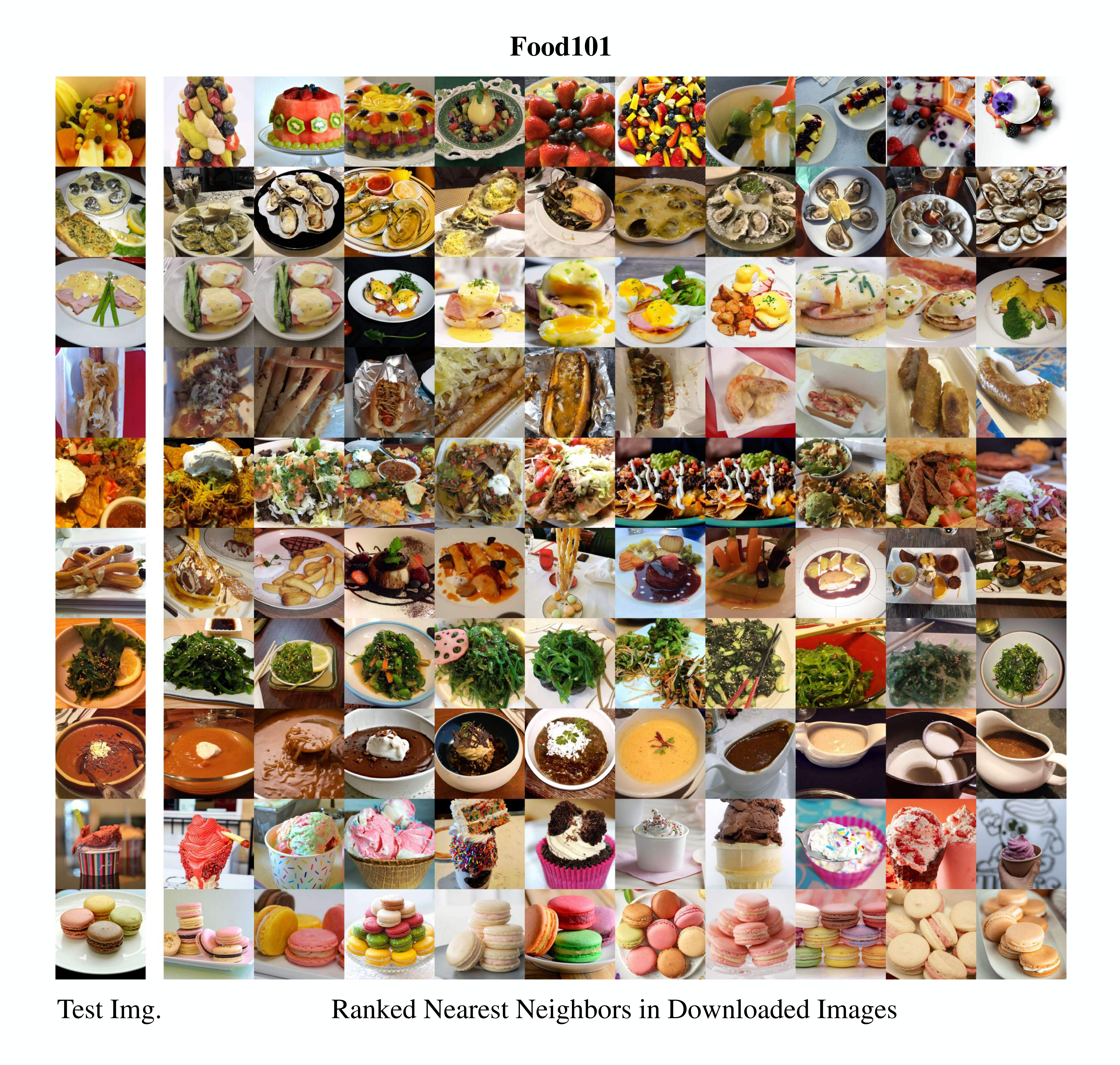}
    \caption{\textbf{Top-10 most similar online images to Food101}}
    \label{fig:internet-nns:food}
\end{figure}

\begin{figure}
    \centering
    \includegraphics[width=0.9\linewidth]{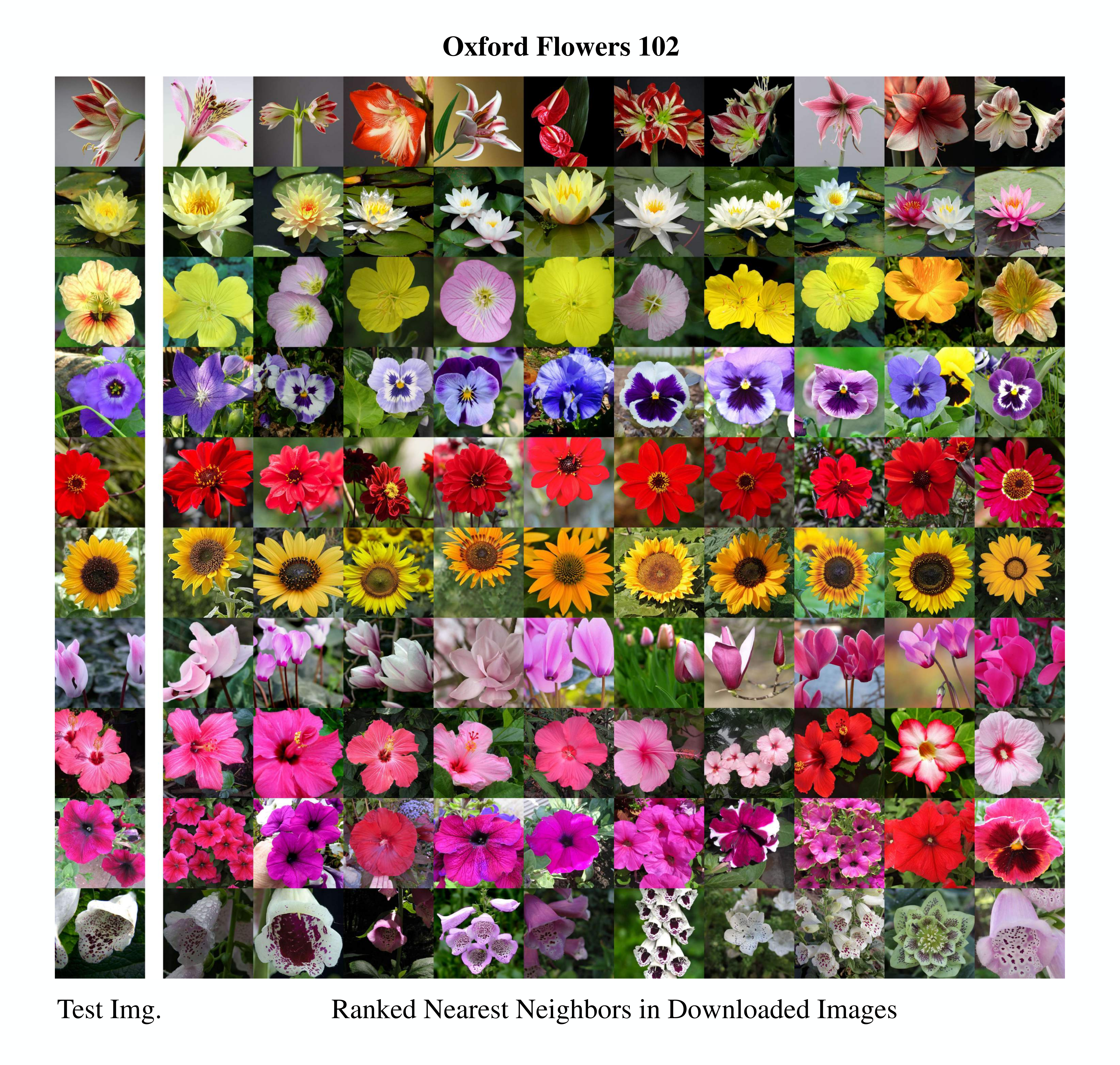}
    \caption{\textbf{Top-10 most similar online images to Flowers102}}
    \label{fig:internet-nns:flowers}
\end{figure}

\begin{figure}
    \centering
    \includegraphics[width=0.9\linewidth]{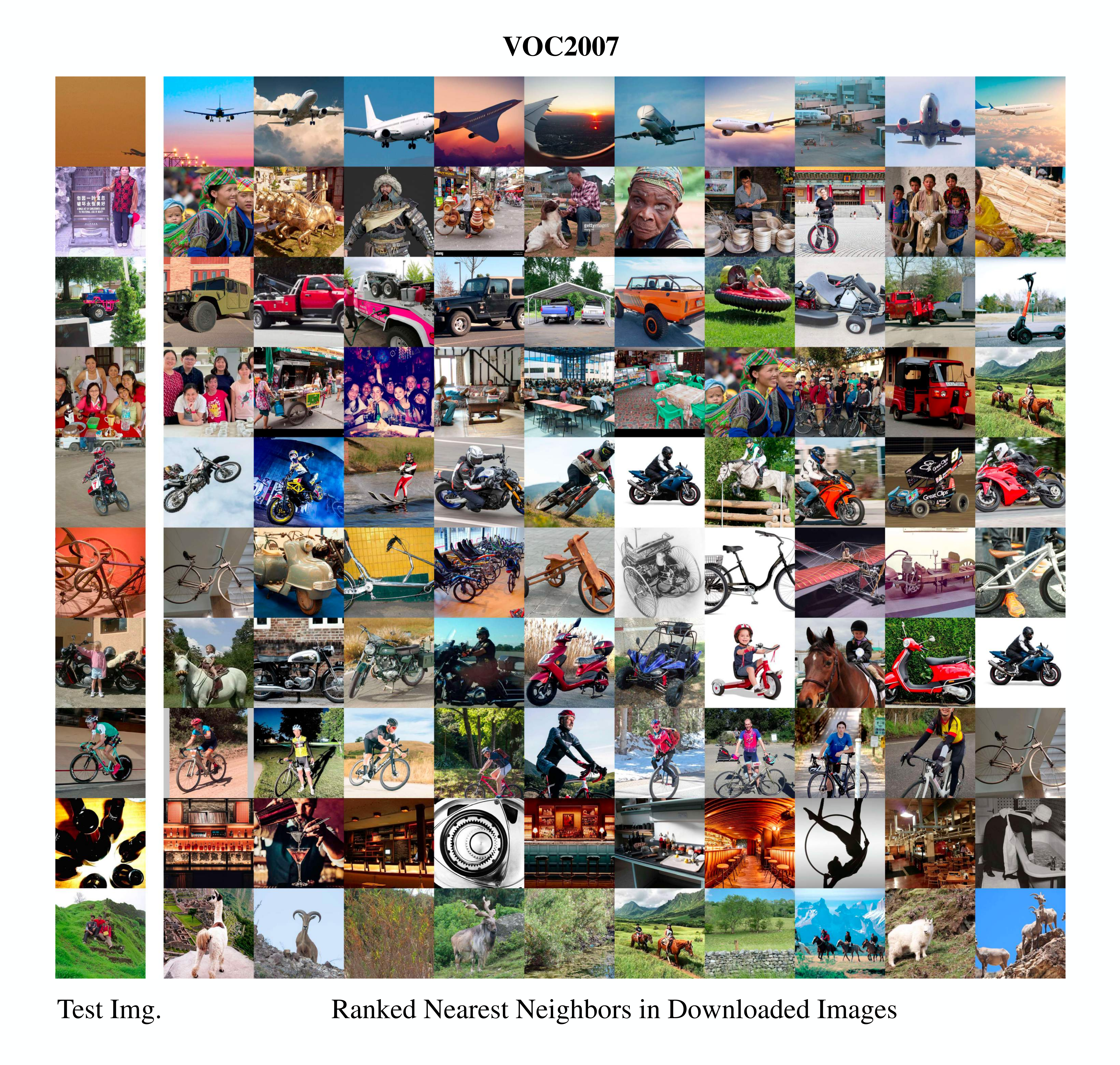}
    \caption{\textbf{Top-10 most similar online images to PASCAL VOC2007}}
    \label{fig:internet-nns:voc}
\end{figure}

\begin{figure}
    \centering
    \includegraphics[width=0.9\linewidth]{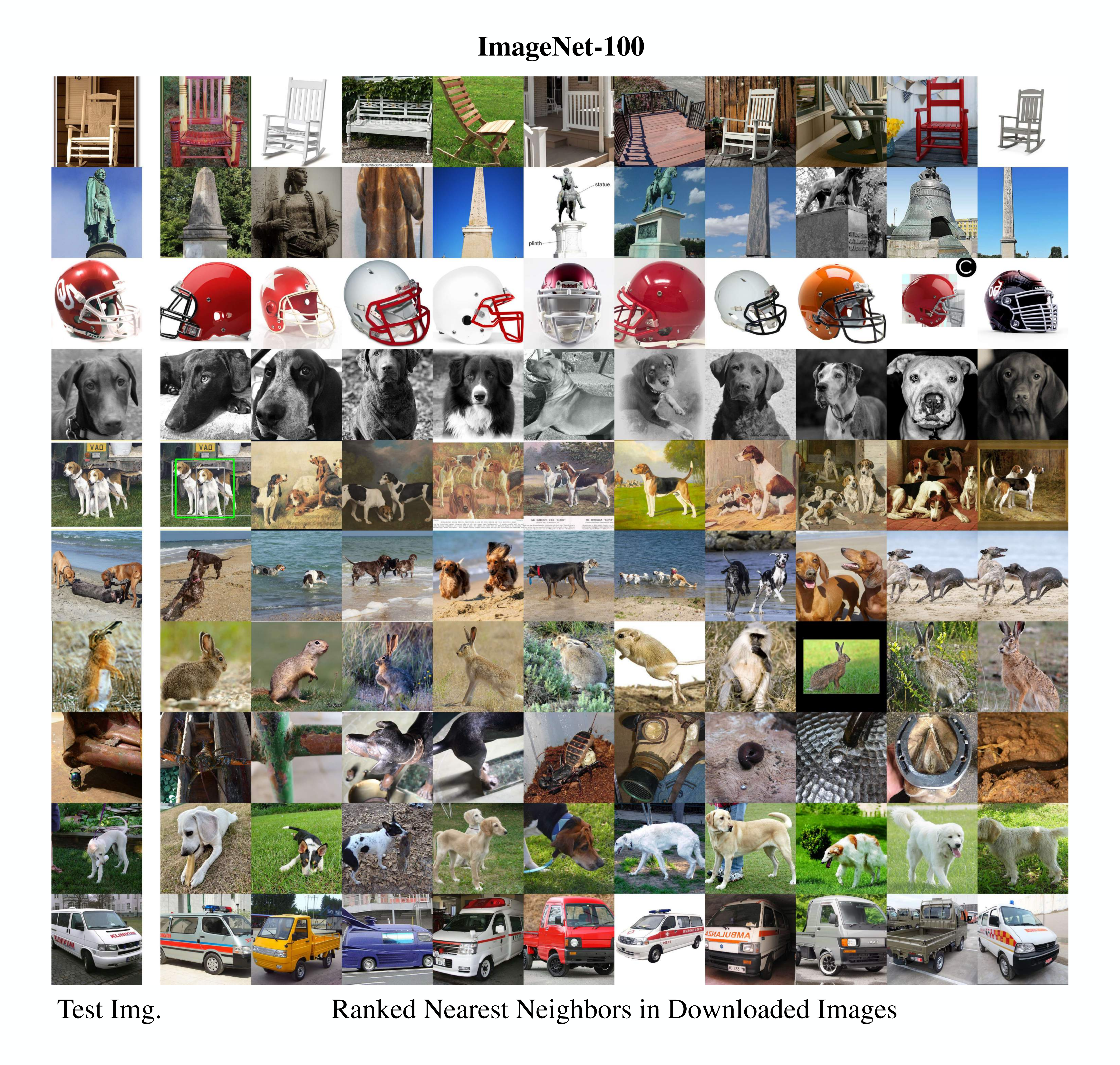}
    \caption{\textbf{Top-10 most similar online images to IN100}}
    \label{fig:internet-nns:in100}
\end{figure}

\section{Method Details}

\subsection{WordNet Lemmas}
\label{sec:wordnet_lemmas}
We draw our concepts from the WordNet hierarchy \cite{miller1995wordnet}, which consists of $146{,}347$ noun lemmas. For reference, here are 32 randomly sampled concepts:
\begin{quote}
{\tt { 
"resolution",
"lodgment",
"phycobilin",
"acidosis",
"widening",
"human face",
"family Crassulaceae",
"sail",
"Ipomoea imperialis",
"Davis",
"prothrombin",
"cease",
"marsh clematis",
"major power",
"chump change",
"madcap",
"junky",
"pere david's deer",
"make-up",
"genus Rumex",
"gape",
"Brachychiton populneus",
"bell morel",
"wain",
"friendly",
"Principe",
"bottle green",
"glycerol trimargarate",
"water-shield",
"San Joaquin River",
"woodsman",
"pin".
}}
\end{quote}

\subsection{GPT-J Descriptor Prompting}
\label{sec:gptj-descriptors}
We use GPT-J-6B~\cite{gpt-j}, a free, open-source autoregressive language model, to generate useful descriptors for a given concept. We use the following prompt template: 
\begin{itemize}
    \item[] \texttt{"What are some words that describe the quality of `\{concept\}'?} 
    \item[] \texttt{The \{concept\} is frail.}
    \item[] \texttt{The \{concept\} is red.}
    \item[] \texttt{The \{concept\} is humongous.}
    \item[] \texttt{The \{concept\} is tall.}
    \item[] \texttt{The \{concept\} is"}
\end{itemize}

We sample completions with a temperature of 0.9 and a max length of 100 tokens. We truncate the completion after the first comma, period, underscore, or newline character (including the special character). If the truncated completion is degenerate and contains a duplicate of the concept, we resample another completion. After successfully sampling a descriptor, we prepend it to the concept and use the resulting phrase as our search query.

For reference, here are 32 randomly sampled descriptors for ``labrador retriever'':
\begin{quote}
{\tt { 
"a good-looking dog",
"very gentle",
"a",
"brown",
"lovable",
"a strong runner",
"a male or a female",
"sturdy",
"agile",
"a strong",
"beautiful",
"a male",
"kind",
"long-haired",
"a male or a female",
"a good-looking dog",
"gentle",
"medium",
"loyal",
"very gentle",
"blue-eyed",
"sturdy",
"blue-eyed",
"a retriever",
"kind",
"loyal",
"large",
"brown",
"good-natured",
"gentle",
"large",
"small".
}}
\end{quote}

\subsection{Concept Vocabulary Size}
\label{sec:concept_vocab_size}
As stated in Section \ref{subsec:text_query_generation}, our vocabulary comprises the $146{,}347$ noun lemmas in the WordNet hierarchy. Thus, in all our experiments, Internet Explorer only searches for WordNet terms (plus the class names, if we have knowledge of the label set). We found that this worked quite well for these standard benchmarks. Note that expanding the vocabulary (e.g., adding technical terms relevant to a specific topic) can easily be done by adding those terms to the list of possible concepts. One easy extension would be to add page titles and frequent unigrams and bigrams from Wikipedia, as was done to generate the CLIP training set~\cite{radford2021learning}. Doing so would expand our vocabulary to roughly $500{,}000$ total concepts. 

\subsection{Query Model Details}
\label{sec:query_model_details}
\paragraph{Temperature for concept distribution}
After estimating scores $r(c_i)$ for each concept $c_i$, we do a temperature-scaled softmax, followed by the tiering operation described in Section 2.6. We compute the temperature $\tau$ such that 
\begin{align}
     \text{SMR} = \frac{\max_i r(c_i) - \min_i r(c_i)}{\tau}
\end{align}
where the ``softmax range'' $\text{SMR} \in \mathbb R$ is the desired gap between the largest and smallest scores after temperature scaling. After the softmax $p(c_i) \propto \exp(r(c_i) / \tau)$, the softmax range determines the likelihood ratio of most likely concept to least likely concept: 
\begin{align}
    \frac{\max_i p(c_i)}{\min_i p(c_i)} &= \frac{\max_i \exp(r(c_i) / \tau)}{\min_i \exp(r(c_i) / \tau)} \\
      &= \exp \left(\frac{\max_i r(c_i) - \min_i r(c_i)}{\tau}\right) \\
    &= \exp(\text{SMR})
\end{align}
Thus, SMR is an easy way to specify the relative likelihood of the highest and lowest scoring concepts and achieve a desired exploration-exploitation balance.

\paragraph{Label set-guided vocabulary}
To reduce our search space in the label set-guided setting, in which we know the English names of the classes a priori, we generate a subset of the WordNet vocabulary that contains only the top-$10\%$ most semantically-relevant concepts to each target dataset.
We use a pre-trained text embedding model~\cite{reimers2019sentence} to generate $384$-dimensional embeddings for each concept in WordNet, using the same template described in Section 2.5 of the main paper: %

\begin{quote}
\vspace{-0.07in}
{\tt {\small \{lemma\} (\{hypernym\}): \{definition\}}}.
\vspace{-0.07in}
\end{quote}

To generate a similar embedding for concepts in target datasets, we use the summary from Wikipedia in place of the definition and the ``category'' of the target dataset (shown in \cref{tab:dataset_categories}) in place of the hypernym:

\begin{quote}
\vspace{-0.07in}
{\tt {\small \{label\} (\{category\}): \{summary\}}}.
\vspace{-0.07in}
\end{quote}

\begin{table}
    \centering
    \begin{tabular}{ll}
    \toprule
        Dataset & Category \\
    \midrule
        Oxford Flowers102 & Flower \\
        Oxford IIIT Pets & Pet \\
        Food101 & Food \\
        Birdsnap & Bird \\
        VOC2007 & Object \\
    \bottomrule
    \end{tabular}
    \caption{\textbf{Target Dataset ``Category''}.
    }
    \label{tab:dataset_categories}
\end{table}

After generating the embeddings for each concept in the target dataset, we find the $k$-NN distance for each WordNet concept to the target dataset embeddings, where $k$ is chosen to be $1/3$ the size of the class label set.
We then rank the concepts in WordNet by the distance and take the closest $10\%$ of terms as our subset. This subset is used for all methods in the label set-guided setting, including the random exploration methods.

\subsection{Training Details}
In each iteration, we download roughly 25k candidate images, since we download up to 100 images for each of the 256 queries. Given this set $\mathcal C$ of candidate images, we sample $\text{PCR} \times |\mathcal C|$ images from the union of the replay buffer $\mathcal B$ and the target dataset training images $\mathcal D$. PCR (past data to candidate data ratio) is a scalar value that determines how much old data vs new data to train on at every iteration. We set $\text{PCR}=2$ for all experiments. We perform $10$ epochs of training over the union of the new candidate data and the sampled replay buffer and target dataset images. 

\subsection{Hyperparameters}

\cref{tab:hyperparameters} shows our hyperparameter values, which are shared across datasets. We perform minimal hyperparameter tuning and copy most of the values from the MoCo-v3~\cite{chen2021empirical} ResNet-50 configuration. 
Our code has been released at \href{https://github.com/internet-explorer-ssl/internet-explorer}{\url{https://github.com/internet-explorer-ssl/internet-explorer}}, which we hope will clarify any remaining implementation details and make it easy for the community to reproduce and build on our work. 
\begin{table}
    \centering
    \begin{tabular}{ll}
    \toprule
        Hyperparameter & Value \\
    \midrule
        Architecture & Resnet-50 \cite{he2016deep} \\
        Optimizer & LARS \cite{you2017large} \\
        Batch size & $224$ \\
        Learning rate & $0.8 \times \frac{224}{256}$ \\
        Learning rate schedule & constant \\
        MoCo momentum & $0.9985$ \\
        RandomResizedCrop min crop area & $0.2$ \\
        Queries per iteration & $256$ \\
        Requested images per query & $100$ \\
        Min images per query & $10$ \\    
        Softmax range (SMR) & $3$ \\
        PCR & $2$ \\
        Epochs per iteration & $10$ \\
    \bottomrule
    \end{tabular}
    \caption{\textbf{Internet Explorer hyperparameters}.}
    \label{tab:hyperparameters}
\end{table}

\subsection{Image Licenses}
Internet Explorer uses images that were indexed by a web crawler (Google Images and LAION) or uploaded to Flickr. The images and their rights belong to their respective owners; we use, download, and train on them under fair use guidelines for research.

\subsection{Domain Dataset Descriptor Details}
\label{sec:fmow_wilds_details}
When targeting a niche domain dataset---in which a practitioner almost surely has useful a priori knowledge to impart---it is simple to modify Internet Explorer to accelerate learning. Rather than using GPT to help provide variety to our queries for a concept, we can use leverage our practitioner's domain knowledge to help hone our search from the start.

This amounts to defining a list of ``descriptors'' that help return relevant results for arbitrary queries.
For example, the below list of 16 descriptors was selected for the FMoW satellite dataset to help return satellite imagery when prepended to concepts (\eg, ``tennis court'') instead of their more canonical views.
This list was hand-selected through trial \& error using a variety of randomly selected concepts.
Note that this static list replaces the GPT-J generated descriptors for this dataset. 

FMoW-WILDS Descriptors:
\begin{quote}
{\tt {
    "a centered satellite photo of",
    "a satellite photo of",
    "a google earth photo of",
    "satellite view of",
    "high resolution satellite",
    "high resolution satellite imagery of",
    "aerial satellite",
    "aerial satellite view",
    "aerial satellite view of",
    "satellite imagery, centered photo of",
    "satellite imagery, photo of",
    "military highest resolution satellite imagery of",
    "NASA imagery of",
    "geo high resolution satellite",
    "land cover satellite image of",
    "european satellite close up aerial image of",
    "super high resolution highest resolution satellite imagery"
}}
\end{quote}

\section{Proof of \cref{lemma:speedup}}
\label{sec:proof}
Here, we prove \cref{lemma:speedup} from \cref{subsec:provable_speedup}, which we repeat below: 
\lemmaspeedup*

\begin{proof}
This problem is a variant of the coupon collector problem. Let's first compute $T_{base}$ as the sum of expected times $t_i$ to identify the next relevant concept. 
\begin{align}
    T_{base} &= \sum_{i=1}^{cs} t_i \\
             &= \sum_{i=1}^{cs} \frac{1}{p_i} \\
             &= \sum_{i=1}^{cs} \frac{n}{cs + 1 - i} \\
             &= n \sum_{i=1}^{cs} \frac{1}{cs + 1 - i} \\
             &= n H_{cs}
\end{align}
where $H_{cs}$ is the $cs$th harmonic number. Similarly, we can compute $T_{GPR}$ as the sum of expected times $t_i$ to identify the next relevant cluster.  
\begin{align}
    T_{GPR} &= \sum_{i=1}^{c} t_i \\
             &= \sum_{i=1}^{c} \frac{1}{p_i} \\
             &= \sum_{i=1}^{c} \frac{n}{s (c + 1 - i)} \\
             &= \frac{n}{s} \sum_{i=1}^{c} \frac{1}{c + 1 - i} \\
             &= \frac{nH_{c}}{s}
\end{align}
The speedup is then $\frac{T_{base}}{T_{GPR}} = s \frac{H_{cs}}{H_c} \approx s \log s$.
\end{proof}

We find that in practical settings (e.g., the Pets example analyzed in \cref{fig:reward_over_training}), we can accurately predict how many samples are required to discover all useful concepts. If the vocabulary size is $n \approx 150{,}000$, the number of clusters is about $c = 2$ (one for cats and one for dogs), and the size of each cluster is about $150$, then $T_{GPR} = 1500$, which roughly matches the $9$ iterations $\times 256$ queries/iteration $= 1792$ queries it took to discover both cats and dogs in the Pets dataset.

\section{Progression of downloaded images}
\label{sec:progression_downloaded_imgs}
Just as \cref{fig:progression} in the main paper showed how Internet Explorer progressively discovers useful data when targeting the Pets dataset, \cref{fig:birdsnap_progression,fig:flowers_progression,fig:food_progression,fig:voc_progression} show the progression of downloaded images when targeting Birdsnap, Flowers, Food, and VOC respectively. Note that this analysis is in the self-supervised setting, where Internet Explorer has no knowledge of the label set. Thus, it is quite surprising that Internet Explorer is able to identify relevant images in so few iterations. 

\begin{figure*}[b]
    \centering
    \includegraphics{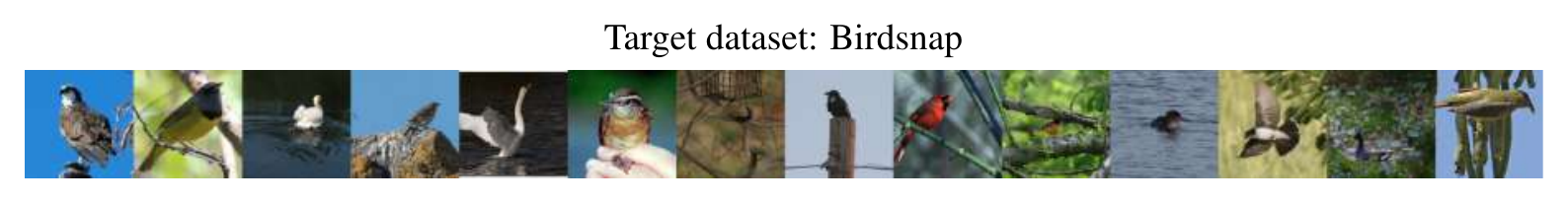} \\
    \vspace{-0.8em}
    \includegraphics{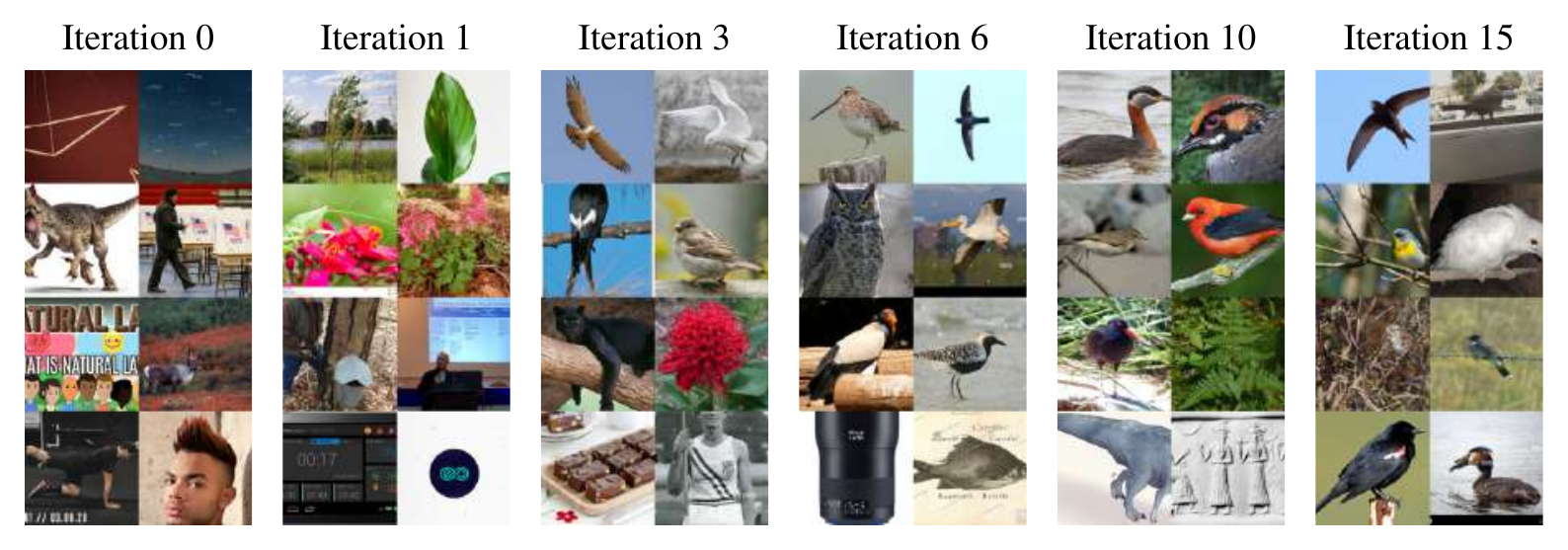}
    \caption{\textbf{Progression of downloaded Birdsnap images.} This corresponds to Ours++ without using label set information. }
    \label{fig:birdsnap_progression}
\end{figure*}

\begin{figure*}
    \centering
    \includegraphics{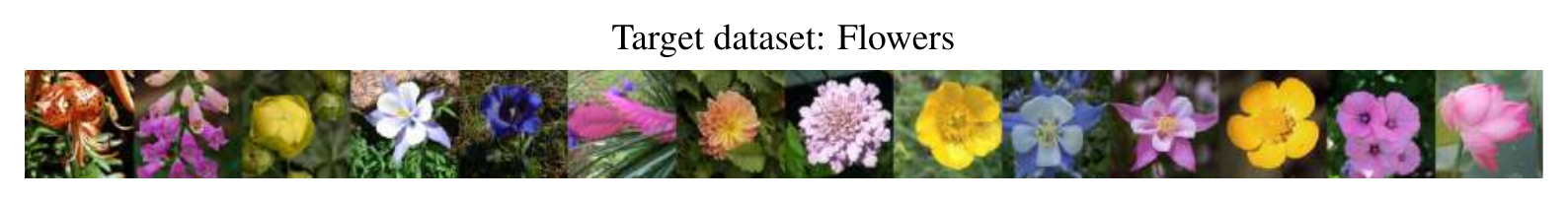} \\
    \vspace{-0.8em}
    \includegraphics{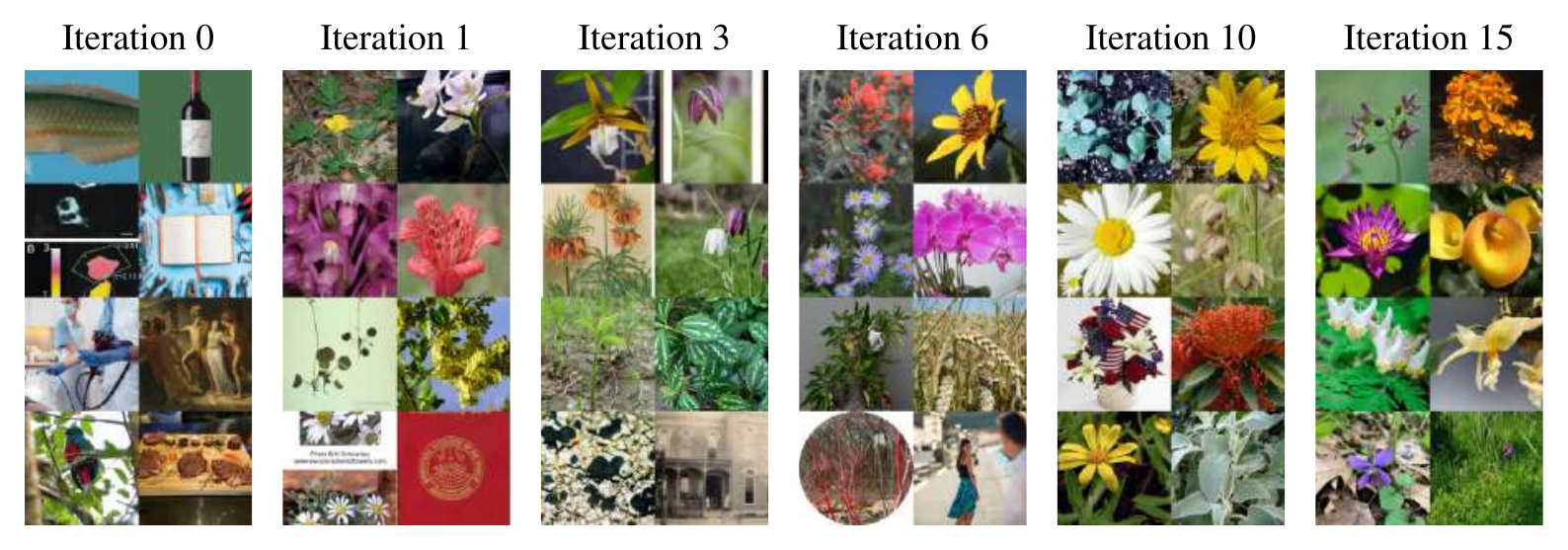}
    \caption{\textbf{Progression of downloaded Flowers images.} This corresponds to Ours++ without using label set information. }
    \label{fig:flowers_progression}
\end{figure*}

\begin{figure*}
    \centering
    \includegraphics{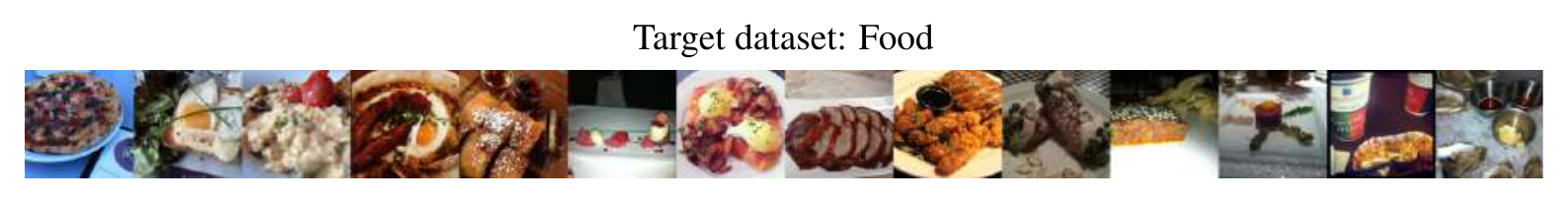} \\
    \vspace{-0.8em}
    \includegraphics{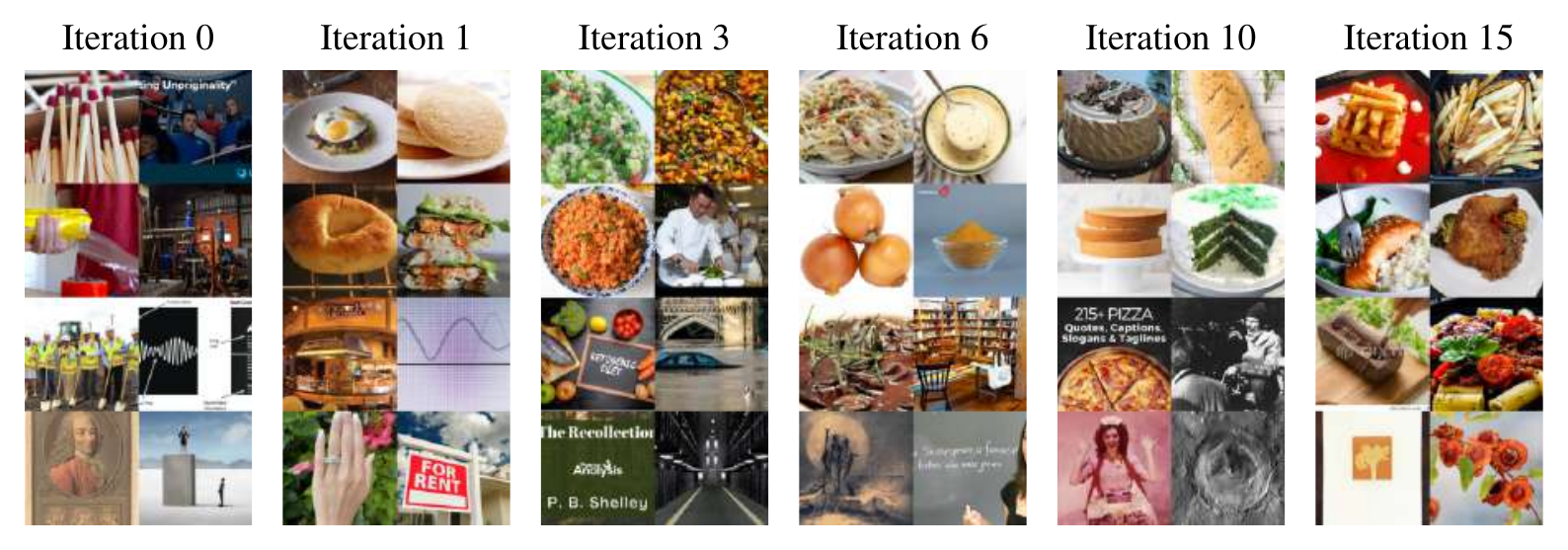}
    \caption{\textbf{Progression of downloaded Food images.} This corresponds to Ours++ without using label set information. }
    \label{fig:food_progression}
\end{figure*}

\begin{figure*}
    \centering
    \includegraphics{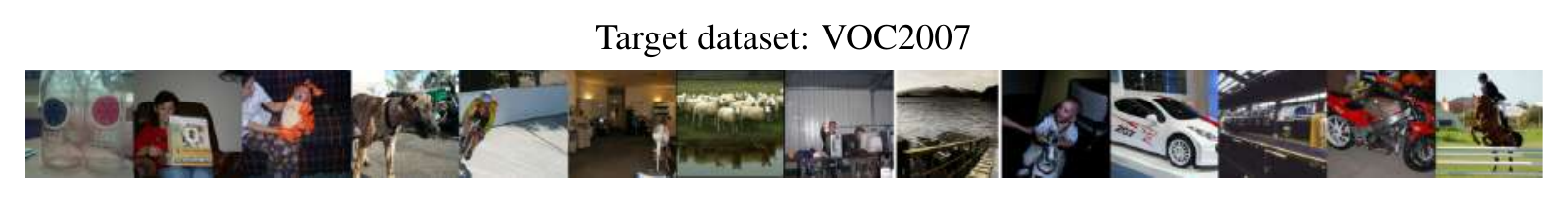} \\
    \vspace{-0.8em}
    \includegraphics{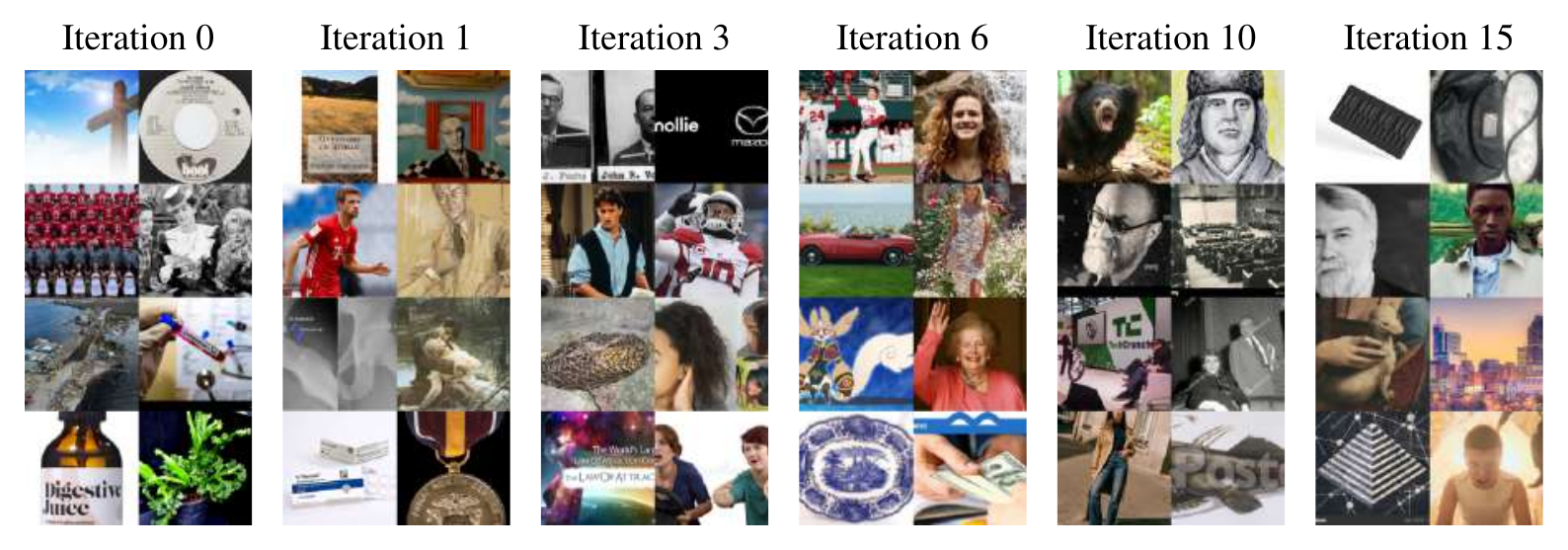}
    \caption{\textbf{Progression of downloaded VOC2007 images.} This corresponds to Ours++ without using label set information. }
    \label{fig:voc_progression}
\end{figure*}

\end{document}